\documentclass[11pt]{article}
\usepackage[letterpaper, margin=1in]{geometry}
\usepackage{amsmath,amsthm}
\usepackage{newtx}
\usepackage{graphicx}
\usepackage{caption} 
\usepackage{xcolor}
\definecolor{linkcolor}{RGB}{40,70,120}
\usepackage[
    colorlinks=true,
    allcolors=linkcolor
]{hyperref}
\usepackage{authblk}
\usepackage{algorithm, algorithmic}
\usepackage[round, authoryear]{natbib}
\usepackage{cleveref}

\theoremstyle{plain}
\newtheorem{theorem}{Theorem}
\newtheorem{lemma}[theorem]{Lemma}
\newtheorem{corollary}[theorem]{Corollary}
\newtheorem{proposition}{Proposition}
\theoremstyle{definition}
\newtheorem{definition}{Definition}
\newtheorem{assumption}{Assumption}
\theoremstyle{remark}
\newtheorem{remark}{Remark}

\newcommand{\acks}[1]{\section*{Acknowledgments}#1}

\begin{document}

\title{Blackwell Approachability and Gradient Equilibrium are Equivalent}
\author[1]{Brian W. Lee}
\author[1]{Nika Haghtalab}
\author[1, 2]{Michael I. Jordan}
\author[1]{Ryan J. Tibshirani}

\affil[1]{University of California, Berkeley}
\affil[2]{Inria \& École Normale Supérieure}
\date{}

\maketitle

\begingroup
\renewcommand{\thefootnote}{}
\footnotetext{Accepted for presentation at the Conference on Learning Theory (COLT) 2026.}
\endgroup

\vspace{-0.8cm}
\begin{abstract}
Gradient equilibrium (GEQ) is a recently introduced online optimization framework that generalizes first-order stationarity from offline optimization and abstracts problems like online conformal prediction. 
While GEQ has curious similarities with known online learning frameworks, namely regret minimization, prior work has shown that GEQ error and regret are incomparable objectives, leaving open a precise understanding of how GEQ fits into the broader online learning landscape.

In this work, we show that GEQ is equivalent to Blackwell approachability in the algorithmic sense. 
That is, a Blackwell approachability problem can always be solved using queries to a black-box GEQ oracle, with no asymptotic loss in the oracle's error rate, and vice versa. 
Taken together with known equivalences between approachability, regret minimization, and calibration, these results imply that GEQ is equivalent to these frameworks, as well. 
Our reductions are efficient and can be used to transfer refined guarantees, such as optimism and strong adaptivity, from regret minimization to GEQ. 
Along the way, we also identify necessary and sufficient conditions for GEQ, and establish reductions between different notions of GEQ with unconstrained and constrained decision sets.
\end{abstract}

\vspace{-0.5cm}
\section{Introduction}\label{sec:intro}
\vspace{-0.3cm}
The many frameworks studied in online learning, from regret minimization to calibration, formalize different notions of what it means to make good decisions against an adaptive adversary. 
Though at first glance, these frameworks appear both syntactically and semantically different from one another, celebrated works have revealed structural connections between them that make oracle reductions possible \citep{FOSTER199740, foster_proof_1999, 2e2b81f3-680a-335e-852c-e26662afca5e, cesa-bianchi_prediction_2006, NIPS2006_f14bc21b, KAKADE2008115, 10.1287/moor.1100.0465, abernethy_blackwell_2011, perchet2013approachability, shimkin_online_2016}. 
These reduction techniques underlie recent algorithmic advances in calibration \citep{pmlr-v80-hebert-johnson18a, foster_forecast_2021, pmlr-v238-okoroafor24a}, multi-objective learning \citep{lee2022online, haghtalab2023unifying}, and related studies of ``universally useful" predictors \citep{okoroafor2025near, balakrishnan2025panprediction}.

\begin{figure*}[t]
    \centering
    \includegraphics[width=1\linewidth]{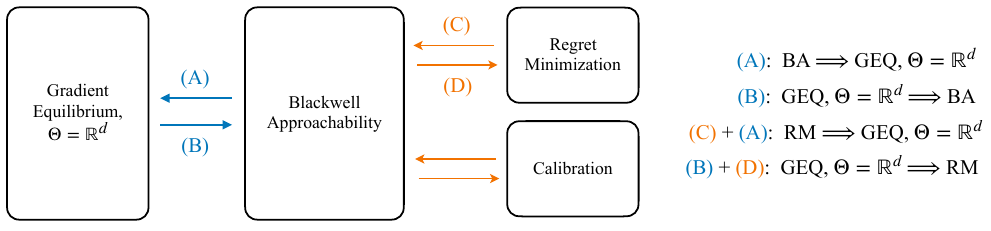}
    \caption{
        Our main contributions are black-box oracle reductions that establish an equivalence between gradient equilibrium (GEQ) and Blackwell approachability (BA); shown in blue. 
        Taken together with known reductions between BA, regret minimization (RM), and calibration (Cal) from \citet{abernethy_blackwell_2011} and \citet{perchet2013approachability}, shown in orange, our results clarify the connections between GEQ, RM, and Cal as well.  
    }
    \label{fig:flow}
\end{figure*} 

Recently, \citet{JMLR:v26:25-0356} introduced a new framework for online optimization.
Instead of selecting decisions to minimize regret relative to the best fixed decision in hindsight---the classical goal of online optimization algorithms---\citet{JMLR:v26:25-0356} asked whether it is possible to make decisions that drive the time-average of the subgradients of losses to zero, and whether this objective enjoys advantages over regret.
Formally, this objective, called \textit{gradient equilibrium} (GEQ), is defined as follows. 
Fix the decision set $\Theta = \mathbb{R}^d$ and a set of loss functions $\mathcal{L}$, where each loss $\ell \in \mathcal{L}$ takes the form $\ell: \Theta \to \mathbb{R}$ and admits a subgradient $g(\theta) \in \partial \ell(\theta)$.\footnote{Throughout Section~\ref{sec:intro}, we assume $\Theta = \mathbb{R}^d$. The general definition of GEQ with $\Theta \subseteq \mathbb R^d$ is given in Section~\ref{sec:def}.}
A sequence of decisions $\{ \theta_t \}_{t \geq 1}$, with $\theta_t \in \Theta$ for all $t \geq 1$, achieves GEQ on the sequence of losses  $\{ \ell_t \}_{t \geq 1}$, with $\ell_t \in \mathcal{L}$ for all $t \geq 1$, if
\[
\left\| \frac{1}{T} \sum_{t=1}^{T} g_t(\theta_t) \right\|_2 \to 0 \quad \text{as} \quad T \to \infty.
\]
GEQ offers a way to lift statistical problems, whose goals are not easily expressed as regret relative to~an~oracle decision, to the online setting. 
For example, when losses are instantiated with pinball losses, the corresponding GEQ problem is online quantile debiasing, which is closely related to online conformal~prediction. 

Though GEQ has curious similarities with known online learning frameworks, it has important differences, and a precise understanding of how GEQ fits into the broader online learning landscape has remained open.
A comparison of GEQ and regret minimization illustrates the tensions. 
To begin, both frameworks lift a first-order condition from offline optimization to the online setting. 
GEQ takes the stationarity condition, $g(\theta^\star) = 0$, which is meaningful for both convex and nonconvex losses, and requires that it is enforced \textit{on average over time} and \textit{in the limit}. 
Analogously, regret minimization (after reducing online convex optimization to online linear optimization) requires that the variational inequality for first-order optimality in offline convex optimization, $\sup_{\theta \in \Theta} \langle g(\theta^\star), \theta^\star - \theta \rangle \leq 0$, holds on average over time and in the limit, 
\[
\sup_{\theta \in \Theta} \frac{1}{T} \sum_{t=1}^{T} \langle g_t(\theta_t), \theta_t - \theta \rangle \to 0 \quad \text{as} \quad T \to \infty.
\]
Despite these parallels, there are notable differences between GEQ and regret minimization.
\begin{itemize}
    \item \textbf{Incomparable objectives}. 
    \citet{JMLR:v26:25-0356} showed that GEQ error and regret are incomparable objectives, in the sense that a sequence of decisions can achieve GEQ on a sequence of losses but not vanishing regret, and vice versa. 
    This is in contrast with how, in offline convex optimization, the stationarity and variational inequality conditions for optimality imply one another. 
    Furthermore, GEQ error is absolute, while regret is defined as error relative to an oracle decision. 
    
    \item \textbf{Incomparable sufficient conditions}. 
    The conditions on the loss class $\mathcal{L}$ and decision set $\Theta$ under which GEQ can be achieved are incomparable to the conditions under which vanishing regret can be achieved. 
    GEQ can be achieved when losses satisfy a regularity condition called restorativity, while regret minimization can be achieved when losses are convex. 
    These classes are incomparable, as there exist nonconvex losses that are restorative and convex losses that are not restorative. 
    Furthermore, GEQ can be achieved when the decision set is unconstrained, whereas regret minimization typically presumes bounded decision sets, as significant care is needed to control regret relative to all of $\mathbb R^d$.
    
    \item \textbf{Algorithmic differences}. 
    Algorithms for GEQ resemble regret minimization algorithms, but differ from them in important ways. 
    For example, online gradient descent (OGD) with a step size that is constant with respect to all problem parameters, including time horizon, is the canonical algorithm for GEQ.
    While OGD is a canonical algorithm for regret minimization as well, its step size must either be selected with knowledge of problem parameters, or tuned adaptively. 
\end{itemize}

\paragraph{Contributions.}
We clarify how GEQ fits into the broader online learning landscape by showing that it is equivalent to Blackwell approachability in the algorithmic sense.
It follows that GEQ is also equivalent to all frameworks known to be equivalent to approachability, such as regret minimization and calibration; see Figure~\ref{fig:flow}.
Hence, while the GEQ guarantee is semantically different from known online learning guarantees, GEQ algorithms are equally powerful as classical regret minimization and calibration algorithms. 
After reviewing relevant frameworks in Section~\ref{sec:prelims}, we present the following results. 
\begin{itemize}
    \item \textbf{Reducing GEQ to BA}. 
    In Section~\ref{sec:ba-to-geq}, we present a black-box oracle reduction that converts any algorithm that solves approachability problems into an algorithm that solves GEQ problems. 
    That is, using $T$ queries to a BA algorithm with error rate $\phi(T)$ alongside elementary operations, we construct a GEQ algorithm with error rate $O(\phi(T))$.
    Along the way, we identify Blackwell's condition---a one-step equalization condition from the approachability literature---as the necessary and sufficient condition for GEQ.
    Previously, only a sufficient condition, restorativity, was known for GEQ. 
    
    \item \textbf{Reducing BA to GEQ}.
    In Section~\ref{sec:geq-to-ba}, we present a black-box oracle reduction that converts any algorithm that solves GEQ problems into an algorithm that solves BA problems. 
    That is, using $T$ queries to a GEQ algorithm with error rate $\phi(T)$ alongside elementary operations, we construct a BA algorithm with error rate $O(\phi(T))$. 
    Along the way, we show how to solve any GEQ problem with a constrained decision set using an algorithm for unconstrained GEQ, which may be of independent~interest. 
    
    \item \textbf{Applications}.
    Taken together with known reductions between approachability, regret minimization, and calibration \citep{abernethy_blackwell_2011, perchet2013approachability}, our reductions enable the use of algorithms in one framework to solve problems in another. 
    In Section~\ref{sec:algo}, we use this insight to derive algorithms for both GEQ and regret minimization.
    First, by plugging existing regret minimization algorithms into one direction of our reduction, we derive GEQ algorithms with optimistic error bounds, which are small when gradients are predictable (\Cref{cor:optim}), and strongly adaptive error bounds, which hold uniformly over all sub-intervals of time (\Cref{cor:sar}).
    Second, by plugging simple GEQ algorithms, such as OGD with constant step size, into our reduction in the other direction, we derive regret minimization algorithms that bypass step size tuning and do not require knowledge of the loss vector norm bound $L$ or time horizon $T$. 
\end{itemize}

\paragraph{Technical overview.}
Our results stem from the elementary observation that GEQ error can be written as
\[
\left\| \frac{1}{T} \sum_{t=1}^{T} g_t(\theta_t) \right\|_2 
= \mathrm{dist}\left( -\frac{1}{T} \sum_{t=1}^{T} g_t(\theta_t), \{0\} \right)
= \max_{\theta \in B_2(1)} \left\langle -\frac{1}{T} \sum_{t=1}^{T} g_t(\theta_t), \theta \right\rangle,
\]
where $\mathrm{dist}(x, S) = \inf_{s \in S}\| x - s\|_2$ denotes the Euclidean distance between a point $x$ and a set $S$ and $B_2(1)$ denotes the Euclidean ball of radius one centered at the origin.

The first equality suggests that a GEQ problem can be interpreted as a Blackwell approachability problem, where negative subgradients serve as vector payoffs and the origin serves as the singleton target set. 
Our reduction from GEQ to approachability follows from formalizing how certain conditions on the GEQ problem ensure that the induced approachability problem can be solved. 

The second equality suggests that GEQ algorithms may have the latent ability to identify dual witnesses of GEQ error, that is, $\theta = \arg\max_{\theta \in B_2(1)} \langle -T^{-1} \sum_{t=1}^{T} g_t(\theta_t), \theta \rangle$.
Our reduction from approachability to GEQ follows from formalizing how this latent ability can be used to identify directions that separate the current time-average of vector payoffs from the target set in a general approachability problem. 

\subsection{Related work}
\paragraph{Equivalences with Blackwell approachability.}
\citet{blackwell_analog_1956}'s celebrated work introduced a model of two-player repeated games with bounded vector-valued payoffs, where the primal player aims to drive the time-average of payoffs into a target set, while the dual player aims to prevent such convergence.
Early work in online learning showed that many fundamental problems, such as regret minimization \citep{2e2b81f3-680a-335e-852c-e26662afca5e, cesa-bianchi_prediction_2006, NIPS2006_f14bc21b} and calibration \citep{foster_proof_1999, 10.1287/moor.1100.0465} can be understood as particular approachability problems, and hence can be solved with approachability algorithms. 
\citet{abernethy_blackwell_2011} later showed that while general approachability problems do not always ``type check" as regret minimization problems, they can nevertheless be solved using regret minimization algorithms, establishing a powerful algorithmic equivalence between approachability and regret minimization. 
\citet{perchet2013approachability} then proved an analogous equivalence between approachability and calibration, although the proposed reductions can be inefficient. 

Algorithmic equivalences with Blackwell approachability have seeded many of the insights that underlie modern advances in calibration \citep{abernethy_blackwell_2011, foster_forecast_2021, pmlr-v238-okoroafor24a}, multi-objective learning \citep{lee2022online, haghtalab2023unifying}, and related areas that formalize notions of ``universally useful" predictors \citep{okoroafor2025near, balakrishnan2025panprediction}.
Concretely, all of these problems require enforcing many desiderata at once, which is naturally formalized through the lens of approachability. 
Algorithmic equivalences, in turn, offer many design paths to algorithms that attain the same worst-case guarantee, but differ in other properties (e.g., runtime) and empirical behavior. 

In the same spirit, we establish an algorithmic equivalence between GEQ and Blackwell approachability and study algorithmic implications for online learning writ large.

\paragraph{Gradient equilibrium}
GEQ was introduced as a framework to lift statistical problems, such as mean debiasing and conformal prediction, to the online setting \citep{JMLR:v26:25-0356}.
\citet{JMLR:v26:25-0356} observed that GEQ error morally resembles regret, but proved that they are incomparable as objectives.
The authors left open the task of formalizing the connection between GEQ and calibration. 
In this work, we show that while GEQ error may be incomparable with existing online learning objectives, GEQ algorithms are equally powerful primitives as regret minimization and calibration algorithms. 

\citet{JMLR:v26:25-0356} also generalized their definition of GEQ to allow for constrained decision sets; \citet{ding2025calibrated} later formalized an alternative notion of GEQ with constraints that more faithfully captures the desiderata of particular applications. 
We show that the former notion of GEQ with constraints is algorithmically equivalent to GEQ without constraints, thereby streamlining the task of designing general GEQ algorithms to that of designing algorithms for the simple, unconstrained case. 

\paragraph{Regret minimization}
Hannan introduced regret to measure deviations from Bayes-optimal play in repeated games \citep{cesa-bianchi_prediction_2006}.
Since then, regret to the best fixed decision in hindsight has been the de facto objective of online optimization, and a rich body of work has developed regret minimization algorithms via connections to convex analysis \citep{orabona_modern_2025}. 
Recent works have studied refined notions of regret that are small when the realized sequence of losses has exploitable structure \citep{rakhlin_optimization_2013}, as well as strongly adaptive guarantees that hold uniformly over all time intervals \citep{10.5555/3045118.3045268, jun_improved_2017}. 
Our results provide a principled way to transfer such refined guarantees to GEQ. 

\section{Models and Preliminaries}\label{sec:prelims}
To begin, we review the Blackwell approachability and gradient equilibrium frameworks.
Additional preliminaries on regret minimization and reductions between approachability and regret minimization from \citet{abernethy_blackwell_2011} are deferred to Appendices~\ref{sec:rm} and~\ref{sec:equiv}.

\paragraph{Notation.}
Let $\mathbb R^d$ for some $d \geq 1$ be the ambient space in what follows. 
Let $B_2(R) \subseteq \mathbb R^d$ be the Euclidean ball of radius $R > 0$ centered at the origin and let $S \subseteq \mathbb R^d$ be a set that contains the origin. 
Let $d(z, S) = \inf_{s \in S} \| z - s \|_2$ be the Euclidean distance between a point $z \in \mathbb R^d$ and $S$, and let $\Pi_S(z) \in \arg\min_{s \in S} \|z - s \|_2$ be the Euclidean projection of $z$ onto $S$.
Let $\sigma_S(z) = \sup_{s \in S} \langle z, s \rangle$ be the support function of $S$. 
Let $\mathrm{cone}(S) = \{ \lambda s: \, \lambda \geq 0, \, s \in S \}$ be the conic hull of $S$ and let $\mathrm{cone} (S)^\circ = \{ y \in \mathbb R^d: \, \langle x, y \rangle \leq 0 \,\,\, \forall x \in \mathrm{cone}(S) \}$ be the polar cone of $S$. 
Let $N_S(s) = \{ g \in \mathbb R^d: \, \langle g, s' - s \rangle \leq 0 \,\,\, \forall s' \in S \}$ and $T_S(s) = N_S(s)^\circ$, respectively, be the normal and tangent cone of $S$ at the point $s \in S$.
Let $\oplus$ denote the vector concatenation operation: for $z_1 \in \mathbb R^{d_1}$ and $z_2 \in \mathbb R^{d_2}$, we have $z_1 \oplus z_2 = (z_1^\top, z_2^\top)^\top \in \mathbb R^{d_1 + d_2}$.
Let $[T]$ denote the ordered set $\{ 1, \dots, T \}$.
Whenever we consider two-player games, we refer to the primal or minimizing player as the ``Learner" and the dual or maximizing player as ``Nature."
Novel results are labeled ``Theorem" or ``Lemma." 
Results from prior work that we restate for completeness are labeled ``Proposition."

\subsection{Blackwell Approachability}
Blackwell approachability (BA) is a framework for repeated two-player games with vector payoffs that extends the scalar minimax theorem.
When payoffs are real-valued, the minimax theorem implies that the Learner can drive the payoff into a set of the form $(-\infty, c]$ in one step if and only if $c$ is at least the minimax value of the game. 
When payoffs are vectors with dimension at least two, simple counterexamples make it impossible for the Learner to drive the payoff into a target set in a single step \citep{abernethy_blackwell_2011}.
Hence, approachability relaxes the Learner's goal to ensuring that the payoff approaches the target set \textit{on average over time} and \textit{in the limit}. 

\begin{definition}[Blackwell approachability]
    A BA problem is a tuple $(\mathbb R^d, A, B, f, S)$, where $A$ and $B$ are the action sets of the Learner and Nature, respectively; $f: A \times B \to \mathbb{R}^d$ is a vector payoff function; and $S \subseteq \mathbb{R}^d$ is the target set. 
\end{definition}

At each time step $t \geq 1$, the Learner selects an action $a_t$, Nature selects an action $b_t$, then the Learner observes the payoff $f(a_t, b_t)$.
A strategy of the Learner in a BA problem $(\mathbb R^d, A, B, f, S)$ is a mapping $\mathcal A: \{ \emptyset \} \cup (\cup_{t \geq 1} (\mathbb R^d)^t) \to A$ that selects action $a_t$ according to
\[
a_t = \mathcal A(f(a_1, b_1), \dots, f(a_{t-1}, b_{t-1})),
\]
with input $\emptyset$ when $t = 1$.
The Learner's goal is formalized as follows. 

\begin{definition}[Approaching target set $S$]
    Let $(\mathbb R^d, A, B, f, S)$ be a BA problem.
    The target set $S$ is approachable if there exists a strategy $\mathcal A$ of the Learner such that for all sequences of Nature's actions $\{ b_t \}_{t \geq 1}$, with $b_t \in B$ for all $t \geq 1$, using $\mathcal{A}$ to select the sequence of actions $\{ a_t \}_{t \geq 1}$, with $a_t \in A$ for all $t \geq 1$, ensures that
    \[
    d\left( \frac{1}{T}\sum_{t=1}^{T} f(a_t, b_t), S \right) \to 0 \quad \text{as} \quad T \to \infty.
    \]
\end{definition}

\citet{blackwell_analog_1956} showed that the following condition is necessary and sufficient for approachability when the target set $S$ is nonempty, closed, and convex.

\begin{proposition}[Blackwell's condition]
    Let $(\mathbb R^d, A, B, f, S)$ be a BA problem, where $S$ is a nonempty closed convex set.
    Then $S$ is approachable if and only if for all $u \in \mathbb{R}^d$, there exists an action of the Learner $a \in A$ such that for all actions of Nature $b \in B$, $\langle u, f(a, b) \rangle \leq \sigma_S(u)$. 
\end{proposition}

Blackwell's condition requires that for any direction $u \in \mathbb R^d$, the Learner can choose an action such that the next payoff is guaranteed to fall in the supporting halfspace of the target set $S$ with normal vector $u$: $H(u) = \left\{ x \in \mathbb R^d: \, \langle u, x \rangle \leq \sigma_S(u) \right\}$.
Throughout the sequel, we make the following assumption. 

\begin{assumption}\label{as:ba}
    The BA problem $(\mathbb R^d, A, B, f, S)$ satisfies the following conditions. 
    For all $a \in A$, $b \in B$, $\| f(a, b) \|_2 \leq L$ for some $L \geq 0$, and $S \subseteq \mathbb R^d$ is a nonempty closed convex set that satisfies Blackwell's condition.
\end{assumption}

The geometric interpretation of Blackwell's condition suggests a two-step strategy, formalized in Algorithm~\ref{alg:proj}. 
At each time step $t$, identify a hyperplane that separates the target set and the current time-average of payoffs, then play an action that ensures the next payoff falls on the correct side of that hyperplane.
Concretely, a separating hyperplane can be identified with the normal vector $u_t = \bar f_{t-1} - \Pi_S(\bar f_{t-1})$, which we interpret as the Learner's approach direction at time $t$. 
Given an approach direction, the Learner can select an action $a_t$ by solving a scalar minimax problem with payoff $\langle u_t, f(a, b) - \Pi_S(\bar{f}_{t-1}) \rangle$.
Blackwell's condition guarantees that this problem has a solution, and we abstract the computation into a halfspace oracle. 

\begin{definition}[Halfspace oracle]
    Let $(\mathbb R^d, A, B, f, S)$ be a BA problem that satisfies Assumption~\ref{as:ba}. 
    A \emph{halfspace oracle}, denoted by $\mathcal O_H$, is any algorithm that, given  $u \in \mathbb R^d$ that identifies a supporting halfspace $H(u)$ of $S$, returns $a = \mathcal O_H(u)$ such that $f(a, b) \in H(u)$ for all $b \in B$.
\end{definition}

\begin{algorithm}[t]
\caption{Blackwell's algorithm}
\label{alg:proj}
\begin{algorithmic}[1]
\REQUIRE Action set $A$, target set $S$, and halfspace oracle $\mathcal O_H$
\STATE Initialize $\bar f_0 = 0$ and arbitrary $a_0 \in A$
\FOR{$t = 1,2,\dots$}
    \IF{$\bar f_{t-1} \in S$}
        \STATE Play action $a_t = a_0$
    \ELSE
        \STATE Set approach direction $u_t = \bar f_{t-1} - \Pi_S(\bar f_{t-1})$ 
        \STATE Play action $a_t = \mathcal O_H(u_t)$
    \ENDIF
    \STATE Observe payoff $f(a_t, b_t)$
    \STATE Update $\bar f_t = \frac{1}{t} \sum_{s=1}^{t}{f(a_s, b_s)}$
\ENDFOR
\end{algorithmic}
\end{algorithm}

\begin{proposition}[Blackwell's algorithm \citep{cesa-bianchi_prediction_2006}]\label{prop:proj}
    Let $(\mathbb R^d, A, B, f, S)$ be a BA problem that satisfies Assumption~\ref{as:ba} and let $\mathcal O_H$ be a valid halfspace oracle. 
    Fix any $T \geq 1$.
    Then for any sequence of Nature's actions $b_1, \dots, b_T \in B$, selecting actions $a_1, \dots, a_T \in A$ using Algorithm~\ref{alg:proj} guarantees that
    \[
    d\left( \frac{1}{T} \sum_{t=1}^{T} f(a_t, b_t), S \right) \leq \frac{2L}{\sqrt{T}}.
    \]
\end{proposition}

More generally, we define two notions of a black-box Blackwell approachability oracle that abstracts the projection step of Algorithm~\ref{alg:proj} and selects approach directions. 

\begin{definition}[BA oracle]
    A BA oracle, denoted by $\mathbf{BA}$, is a family of strategies indexed by valid BA problems, with the following form and guarantee. 
    Let $(\mathbb R^d, A, B, f, S)$ be a BA problem that satisfies Assumption~\ref{as:ba}. 
    Fix any $T \geq 1$.
    For this problem, at each time $t \in [T]$, $\mathbf{BA}$ takes as input the history $f(a_1, b_1), \dots, f(a_{t-1}, b_{t-1})$ and returns an approach direction $u_t \in \mathbb R^d$, with the following guarantee. 
    For any sequence of Nature's actions $b_1, \dots, b_T \in B$, selecting actions $a_1, \dots, a_T \in A$ such that $\langle u_t, f(a_t, b_t) \rangle \leq \sigma_S(u_t)$ for all $t \in [T]$ guarantees that 
    \[
    d\left( \frac{1}{T} \sum_{t=1}^{T} f(a_t, b_t), S \right) \leq \frac{\mathbf{BA}\mathrm{Err}(L, T)}{T}
    \]
    where $\mathbf{BA}\mathrm{Err}(L, T) = o(T)$ is the oracle's promised rate.
    We assume that this rate is non-decreasing in $T$. 
\end{definition}

Hence, if at all time steps, a valid halfspace oracle $\mathcal O_H$ is used to map the approach direction $u_t$ produced by $\mathbf{BA}$ to an action $a_t \in A$, then $\mathbf{BA}$ ensures that the underlying BA problem is solved at rate $\mathbf{BA}\mathrm{Err}$.
We are also interested in settings where a halfspace oracle occasionally makes mistakes and produces ``bad" actions $a_t$ for which the undesired inequality $\langle u_t, f(a_t, b_t) \rangle > \sigma_S(u_t)$ holds. 
Below, we define a notion of a robust BA oracle whose error rate is preserved as long as the number of rounds with an error is bounded.
\begin{definition}[Robust BA oracle]
    A robust BA oracle, denoted by $\overline{\mathbf{BA}}$, is a family of strategies indexed by valid BA problems, with the following form and guarantee.
    Let $(\mathbb R^d, A, B, f, S)$ be a BA problem that satisfies Assumption~\ref{as:ba}. 
    Fix any time $T \geq 1$. 
    For this problem, at each time $t \in [T]$, $\overline{\mathbf{BA}}$ takes as input the history $f(a_1, b_1), \dots, f(a_{t-1}, b_{t-1})$ and returns an approach direction $u_t \in \mathbb R^d$, with the following guarantee. 
    For any sequence of Nature's actions $b_1, \dots, b_T \in B$, selecting actions $a_1, \dots, a_T \in A$ such that $\langle u_t, f(a_t, b_t) \rangle \leq \sigma_S(u_t)$ holds for all but $m$ rounds guarantees that
    \[
     d\left( \frac{1}{T} \sum_{t=1}^{T} f(a_t, b_t), S \right) \leq C_m \cdot \frac{\overline{\mathbf{BA}}\mathrm{Err}(L, T)}{T},
    \]
    where $\overline{\mathbf{BA}}\mathrm{Err}(L, T) = o(T)$ is the oracle's promised rate if $m = 0$ and $C_m \geq 1$ is a constant that depends only on $m$ and satisfies $C_0 = 1$.
    We assume that this rate is non-decreasing in $T$. 
\end{definition}

Blackwell's algorithm is robust in this sense, with constant $C_m = \sqrt{1+2m}$. 
The proof of this observation is an elementary modification of the standard square potential analysis and is given in Appendix~\ref{sec:proj-rob}.
\begin{lemma}\label{lem:proj-rob}
    Let $(\mathbb R^d, A, B, f, S)$ be a BA problem that satisfies Assumption~\ref{as:ba}. 
    Fix any $T \geq 1$. 
    Suppose the approach directions $u_1, \dots, u_T$ are selected as in Algorithm~\ref{alg:proj}.
    Then for any sequence of Nature's actions $b_1, \dots, b_T \in B$, selecting actions $a_1, \dots, a_T \in A$ such that $\langle u_t, f(a_t, b_t) \rangle \leq \sigma_S(u_t)$ holds for all but $m$ rounds guarantees that
    \[
    d\left( \frac{1}{T}\sum_{t=1}^{T} f(a_t, b_t), S \right) \leq \sqrt{1+2m} \cdot \frac{2L}{\sqrt{T}}.
    \]
\end{lemma}

\subsection{Gradient Equilibrium}\label{sec:def}
Gradient equilibrium (GEQ) generalizes first-order stationarity from offline optimization to the online setting, and offers an alternative objective to regret. 
Below, we define GEQ in the general setting where the decision set $\Theta \subseteq \mathbb R^d$ may be constrained. 
However, the GEQ condition is most naturally interpreted in the special case where the decision set $\Theta = \mathbb R^d$ is unconstrained, and all of our equivalence results are formulated in terms of this special case of GEQ. 

\begin{definition}[Gradient equilibrium]
    A GEQ problem is a tuple $(\mathbb R^d, \Theta, \mathcal{G})$, where $\Theta \subseteq \mathbb{R}^d$ is a decision set and $\mathcal{G}$ is a class of vector fields $g: \Theta \to \mathbb R^d$.
\end{definition}
The class $\mathcal G$ may, but need not, be defined as the set of subgradients induced by a class of loss functions.

At each time step $t \geq 1$, the Learner selects a decision $\theta_t \in \Theta$, Nature selects a vector field $g_t \in \mathcal G$, then the Learner observes the vector $g_t(\theta_t)$.
A strategy for the Learner in a GEQ problem $(\mathbb R^d, \Theta, \mathcal G)$ is a mapping $\mathcal A: \{ \emptyset \} \cup (\cup_{t \geq 1} (\mathbb R^d)^t) \to \Theta$ that selects decision $\theta_t \in \Theta$ according to 
\[
\theta_t = \mathcal A(g_1(\theta_1), \dots, g_{t-1}(\theta_{t-1})),
\]
with input $\emptyset$ when $t = 1$. 
The Learner's goal is formalized as follows. 

\begin{definition}[Achieving GEQ]\label{def:geq}
    Let $(\mathbb R^d, \Theta, \mathcal{G})$ be a GEQ problem. 
    The sequence of decisions $\{ \theta_t \}_{t \geq 1}$, with $\theta_t \in \Theta$ with $t \geq 1$, is said to achieve gradient equilibrium on the sequence of vector fields $\{ g_t \}_{t \geq 1}$, with $g_t \in \mathcal{G}$ for all $t \geq 1$, if there exists a sequence of normal vectors $\{ n_t \}_{t \geq 1}$, with $n_t \in N_\Theta(\theta_t)$ for all $t \geq 1$, such that
    \[
    \left\| \frac{1}{T} \sum_{t=1}^{T} g_t(\theta_t) + n_t \right\|_2 \to 0 \quad \text{as} \quad T \to \infty.
    \]
\end{definition}

Suppose $\mathcal G$ is the set of subgradients induced by a class of loss functions.
In the special case where $\Theta = \mathbb R^d$, the normal cone $N_\Theta(\theta) = \{0\}$ for all $\theta \in \Theta$, so the GEQ condition can be interpreted as the convergence of the time-average of the subgradients to the origin.
This is a natural analog of the first-order stationarity condition from unconstrained offline optimization in the online setting.
In the general case where $\Theta \subseteq \mathbb R^d$, the time-average of subgradients may be offset by the time-average of a sequence of normal vectors $\{ n_t \}_{t \geq 1}$.
While this is a natural way to lift the stationarity condition in constrained offline optimization, $-g(\theta^\star) \in N_\Theta(\theta^\star)$, to the online setting, the existential statement about a sequence of normal vectors is unusual. 
For example, the ``there exists" quantifier invoked over normal vectors contrasts the ``for all" quantifier invoked over competitor decisions in the definition of regret. 

A sufficient condition for attaining GEQ is that all gradient fields are bounded in the Euclidean norm and satisfy a restorative property, which ensures that for all decisions with sufficiently large norm, the negative vector points back to the origin. 
Notably, when $\mathcal G$ is the set of subgradients induced by a class of loss functions, vector fields being restorative is incomparable to the underlying loss functions being convex. 
For example, $\ell(\theta) = \langle g, \theta \rangle$ is a linear, hence convex, loss, but satisfies no restorativity condition, as the gradient $g$ is always fixed. 

\begin{definition}[$(h, \phi)$-restorative vector field]
    Suppose $h \geq 0$ and $\phi: \Theta \to \mathbb{R}_+$ is a non-negative function. 
    A vector field $g: \Theta \to \mathbb R^d$ is said to be $(h, \phi)$-\emph{restorative} if
    \[
    \| \theta \|_2 \geq h \implies \langle \theta, -g(\theta) \rangle \leq -\phi(\theta). 
    \]
\end{definition}
We make the following assumption throughout the sequel, and view zero as a constant function.
\begin{assumption}\label{as:restor}
    The GEQ problem $(\mathbb R^d, \Theta, \mathcal G)$ satisfies the following conditions. 
    The decision set $\Theta$ is closed, convex, and contains the origin.
    The vector fields in $\mathcal G$ are $L$-bounded in the Euclidean norm for some $L \geq 0$ and satisfy $(h, 0)$-restorativity for some $h \geq 0$.
\end{assumption}

Under Assumption~\ref{as:restor}, GEQ can be attained by projected online gradient descent (OGD) with a constant step size $\eta = O(1)$. 
Given vectors $g_1(\theta_1), \dots, g_{t-1}(\theta_{t-1})$, projected OGD generates a decision $\theta_t$ according to 
\[
\theta_{t+1} = \Pi_\Theta(\theta_t - \eta g_t(\theta_t)).
\]

\begin{proposition}[Projected OGD \citep{JMLR:v26:25-0356}]\label{prop:ogd}
    Let $(\mathbb R^d, \Theta, \mathcal{G})$ be a GEQ problem that satisfies Assumption~\ref{as:restor}. 
    Fix any $T \geq 1$.
    Then for any sequence of vector fields $g_1, \dots, g_T \in \mathcal G$, selecting decisions $\theta_1, \dots, \theta_T \in \Theta$ using projected OGD with step size $\eta$ and initialization $\theta_0 = 0$ guarantees that there exists a sequence of normal vectors $n_1, \dots, n_T$, with $n_t \in N_\Theta(\theta_t)$ for all $t \in [T]$, such~that
    \[
    \left\| \frac{1}{T} \sum_{t=1}^{T} g_t(\theta_t) + n_t \right\|_2 \leq \sqrt{\frac{L^2}{T} + \frac{2Lh}{\eta T}}.
    \]
\end{proposition}

More generally, we define a black-box GEQ oracle as follows. 
For simplicity, we assume that the error of the oracle only depends on the Lipschitz and restorativity constants of $\mathcal{G}$, not the realized sequence $\{ g_t \}_{t \geq 1}$.

\begin{definition}[GEQ oracle]
    A GEQ oracle, denoted by $\mathbf{GEQ}$, is a family of strategies indexed by valid GEQ problems, with the following form and guarantee.
    Let $(\mathbb R^d, \Theta, \mathcal G)$ be a GEQ problem that satisfies Assumption~\ref{as:restor}. 
    Fix any $T \geq 1$. 
    For this problem, at each time step $t \in [T]$, $\mathbf{GEQ}$ takes as input the history $g_1(\theta_1), \dots, g_{t-1}(\theta_{t-1})$ and returns a decision $\theta_t \in \Theta$, with the following guarantee. 
    For any sequence of vector fields $g_1, \dots, g_T \in \mathcal{G}$, there exists a sequence of normal vectors $n_1, \dots, n_T$, with $n_t \in N_\Theta(\theta_t)$ for all $t \in [T]$, such that
    \[
    \left\| \frac{1}{T}\sum_{t=1}^{T} g_t(\theta_t) + n_t \right\| \leq \frac{\mathbf{GEQ}\mathrm{Err}(L, h, T)}{T},
    \]
    where $ \mathbf{GEQ}\mathrm{Err}(L, h, T) = o(T)$ is the oracle's promised rate.
\end{definition} 

\section{Reducing GEQ to Blackwell Approachability}\label{sec:ba-to-geq}
In this section, we present a black-box oracle reduction from GEQ with $\Theta = \mathbb R^d$ to Blackwell approachability. 
This gives a constructive scheme to convert any BA algorithm into a GEQ algorithm. 

To begin, a GEQ problem $(\mathbb R^d, \mathbb R^d, \mathcal{G})$ can be interpreted as a Blackwell approachability problem $(\mathbb R^d, A, B, f, S)$, where
\begin{align}\label{eq:ba-to-geq-setup}
    A = \mathbb R^d, \quad B = \mathcal{G}, \quad f(\theta, g) = -g(\theta), \quad S = \{ 0 \}.
\end{align}
Since the singleton set containing the origin is nonempty, closed, and convex, this BA problem can be solved if and only if Blackwell's condition holds. 
Below, we show that restorativity implies Blackwell's condition, and then construct an approximate halfspace oracle.

\subsection{Restorativity and Blackwell's Condition}
\begin{lemma}\label{lem:ba}
    Let $(\mathbb R^d, \mathbb R^d, \mathcal G)$ be a GEQ problem that satisfies Assumption~\ref{as:restor}. 
    Then the BA problem $(\mathbb R^d, A, B, f, S)$ defined as in Equation~\ref{eq:ba-to-geq-setup} satisfies Assumption~\ref{as:ba}.
\end{lemma}

\begin{proof}[Proof of Lemma~\ref{lem:ba}]
    We have by assumption that for all $g \in \mathcal G$, $\sup_{\theta \in \mathbb R^d}\| g(\theta)\|_2 \leq L$ and that for all $\theta \in \mathbb R^d$, $\| \theta \|_2 \geq h \implies \langle -g(\theta), \theta \rangle \leq 0$.
    The $L$-boundedness of $f(\theta, g) = -g(\theta)$ is immediate. 
    It remains to show that for all $u \in \mathbb R^d$, there exists $\theta = \theta(u) \in \mathbb R^d$ such that for all $g \in \mathcal G$, $\langle u, -g(\theta) \rangle \leq 0$. 
    
    There are three cases.
    If $u = 0$, then the desired inequality holds trivially, so we can set $\theta(u) = u$. 
    If $h = 0$, then the desired inequality holds for $\theta(u) = u$. 
    Otherwise, we can set $\theta(u) = h \cdot \frac{u}{\|u\|_2}$. 
    Since $\| \theta(u) \|_2 \geq h$ by construction in this case, restorativity implies that for all $g \in \mathcal G$,
    \[
    \langle -g(\theta(u)), u \rangle = \frac{h}{\|u\|_2} \cdot \langle -g(\theta(u)), \theta(u) \rangle \leq 0. \qedhere
    \]
\end{proof}

\begin{remark}
    The observation that Blackwell's condition is necessary and sufficient for GEQ is elementary, but novel. 
    In contrast, restorativity is a sufficient condition developed for the analysis of OGD.
    Nevertheless, there are two reasons why restorativity is an interesting condition that merits further study.
    
    First, restorativity minimizes the restrictions placed on the decision sets $\Theta$ and $\mathcal G$ while ensuring that tractable algorithms exist. 
    That is, while Blackwell's condition is an exact characterization, it is not obvious what conditions on $\Theta$ and $\mathcal G$ enforce it. 
    A natural question is whether Blackwell's condition can be indirectly characterized by response-satisfiability: the dual statement that for all $g \in \mathcal G$, there exists $\theta \in \Theta$ such that $g(\theta) = 0$. 
    However, because the family of mappings $\langle f(\cdot, \cdot), u \rangle: \Theta \times \mathcal G \to \mathbb R$, indexed by $u \in \mathbb R^d$, does not satisfy a minimax theorem, unless $\Theta$ and $\mathcal G$ are severely restricted, response-satisfiability does not readily imply Blackwell's condition. 
    Introducing randomization over $\Theta$ and $\mathcal G$ can recover a minimax theorem, but algorithms derived through this path must optimize over distributions over $\Theta$ and $\mathcal G$, which is generally intractable. 
    In contrast, the algorithms we derive from restorativity are simple and tractable. 

    Second, restorativity is sufficiently expressive for a meaningful equivalence between GEQ and approachability to hold. 
    That is, any algorithm that solves GEQ problems under restorativity is powerful enough to solve any BA problem for which Blackwell's condition holds. 
    A fortiori, any BA problem for which Blackwell's condition holds induces a certain restorative vector field.  
    We discuss details in Section~\ref{sec:geq-to-ba}. 

    Studying other minimal conditions on $\Theta$ and $\mathcal G$ under which GEQ problems can be tractably solved, and an equivalence with approachability preserved, is an interesting direction for future work. 
\end{remark}

\subsection{Approximate Halfspace Oracle}
To make the reduction constructive, we must implement a halfspace oracle for GEQ. 
Unfortunately, given an approach direction $u \in \mathbb R^d$, constructing the decision that Lemma~\ref{lem:ba} prescribes requires knowledge of the restorativity horizon $h$, which we do not assume of the Learner in a GEQ problem.

An elementary workaround is to construct a positive and increasing function $\alpha: \mathbb R_{+} \to \mathbb R_{+}$ and use it to select a decision at each timestep as follows.
At time $t$, receive approach direction $u_t$ from the BA oracle and play the decision $\theta_t = \alpha(t) \cdot \frac{u_t}{\| u_t \|_2}$.
If $\alpha(t) \geq h$, then $\langle -g_t(\theta_t), \theta_t \rangle \leq 0$ must hold.
If $\alpha(t) < h$, then the desired inequality may not hold. 
However, the latter case cannot happen for $t \geq \lceil \alpha^{-1}(h) \rceil$.
For example, if $\alpha(t) = t$, then the proposed decision selection rule can fail to enforce the inequality at most $\lceil h \rceil$ times.
A faster-growing choice of $\alpha$ makes the number of failures arbitrarily small.
For example, if $\alpha(t) = 2^t$, then the proposed decision selection rule can fail to enforce the inequality at most $\lceil \log_2(h) \rceil$ times. 

We say that the above rule is a way to implement an \textit{approximate} halfspace oracle, since it may make a bounded number of errors. 
If a robust approachability oracle is used, then the promised rate is preserved up to a constant that depends only on $\lceil \alpha^{-1}(h) \rceil$.
If the oracle in use does not have a robustness guarantee, then a slightly more involved procedure in which the oracle is reset after a failure of restorativity is witnessed can be used to preserve the oracle's promised rate. 
We present the former procedure as Algorithm~\ref{alg:ba-to-geq} and its guarantee as Theorem~\ref{thm:ba-to-geq}, and defer the statement and analysis of the latter procedure to Appendix~\ref{sec:ba-to-geq-app}.

\begin{algorithm}[t]
\caption{GEQ with Robust BA Oracle}
\label{alg:ba-to-geq}
    \begin{algorithmic}[1] 
    \REQUIRE Positive and increasing function $\alpha: \mathbb R_{+} \to \mathbb R_{+}$ and robust BA oracle $\overline{\mathbf{BA}}$
    \STATE Initialize $f_0 = 0$ and arbitrary $\theta_0 \in \Theta$
    \FOR{$t = 1,2,\dots$}
    \STATE Set approach direction $u_t = \overline{\mathbf{BA}}(f_0, \dots, f_{t-1})$
    \IF{$u_t \neq 0$}
        \STATE Play action $\theta_t =  \alpha(t) \cdot \frac{u_t}{\|u_t\|_2}$
    \ELSE
        \STATE Play action $\theta_t = \theta_0$
    \ENDIF
    \STATE Observe vector $g_t(\theta_t)$ and set $f_t = -g_t(\theta_t)$
    \ENDFOR
  \end{algorithmic}
\end{algorithm}

\begin{theorem}\label{thm:ba-to-geq}
    Suppose $(\mathbb R^d, \mathbb R^d, \mathcal G)$ is a GEQ problem that satisfies Assumption~\ref{as:restor}.
    Fix any $T \geq 1$.
    Then using Algorithm~\ref{alg:ba-to-geq} to select the sequence of actions $\theta_1, \dots, \theta_T \in \mathbb R^d$ guarantees that, for all sequences of vector fields $g_1, \dots, g_T \in \mathcal G$, the inequality $\langle u_t, -g_t(\theta_t) \rangle \leq 0$ holds for all but $m \leq \lceil \alpha^{-1}(h) \rceil$ rounds, and 
    \[
    \left\| \frac{1}{T} \sum_{t=1}^{T} g_t(\theta_t) \right\|_2 
    = d\left( \frac{1}{T} \sum_{t=1}^{T} f(\theta_t, g_t), \{ 0 \} \right)
    \leq C_m \cdot \frac{\overline{\mathbf{BA}}\mathrm{Err}(L, T)}{T}
    \]
    for some constant $C_m \geq 1$ that depends only on $m$ and satisfies $C_0 = 1$.
\end{theorem}

For intuition on this result, suppose $\overline{\mathbf{BA}}$ is instantiated with Blackwell's algorithm, which incurs error $2LT^{-1/2}$ by Proposition~\ref{prop:proj} and has constant $C_m = \sqrt{1+2m}$ by Lemma~\ref{lem:proj-rob}.
Then the bound becomes 
\begin{equation}\label{eq:naive}
    \left\| \frac{1}{T} \sum_{t=1}^{T} g_t(\theta_t) \right\|_2 \leq  O\left( \sqrt{\frac{L^2(1+\lceil \alpha^{-1}(h) \rceil)}{T}} \right),
\end{equation}
which nearly recovers the error rate obtained by OGD with $\eta = 1$. 
After proving the general result, we show that a particular choice of $\alpha$ and a more careful analysis of Blackwell's algorithm for the GEQ problem exactly recovers \citet{JMLR:v26:25-0356}'s analysis of OGD for GEQ.

\begin{proof}[Proof of Theorem~\ref{thm:ba-to-geq}]
    First, because $f(\theta_t, g_t) = -g_t(\theta_t)$ for all $t \in [T]$, we have that
    \[
    \left\| \frac{1}{T} \sum_{t=1}^{T} g_t(\theta_t) \right\|_2 = d\left( \frac{1}{T}\sum_{t=1}^{T} f(\theta_t, g_t), \{0\} \right).
    \]
    
    For all $t \geq \lceil \alpha^{-1}(h) \rceil$, we have $\| \theta_t \|_2 \geq h$, so $\langle -g_t(\theta_t), \theta_t \rangle \leq 0$ must hold. 
    So the desired inequality fails to hold at most $m \leq \lceil \alpha^{-1}(h) \rceil$ times. 
    By the definition of a robust BA oracle, there exists $C_m \geq 1$ that depends only on $m$ and satisfies $C_0 = 1$ such that
    \[
    d\left( \frac{1}{T}\sum_{t=1}^{T} f(\theta_t, g_t), \{0\} \right) \leq C_m \cdot \frac{\overline{\mathbf{BA}}\mathrm{Err}(L, T)}{T}. \qedhere
    \]
\end{proof}

\begin{remark}
    It is easy to see that both OGD and Blackwell's algorithm select approach directions that are aligned with the sum of past negative gradients. 
    Going beyond this, we can exactly recover \citet{JMLR:v26:25-0356}'s original analysis of OGD with step size $\eta$ for GEQ within the approachability framework if we set $\alpha(t) = \eta(t-1)$ and slightly modify Line 3 of Algorithm~\ref{alg:ba-to-geq} to choose $\theta_t = \alpha(t) \cdot u_t$, rather than $\theta_t = \alpha(t) \cdot \frac{u_t}{\| u_t \|_2}$.
    Recalling that Euclidean projections onto the GEQ target set trivially yields $0$, the standard analysis of Blackwell's algorithm yields
    \begin{align*}
        d\left( \frac{1}{T}\sum_{t=1}^{T} f_t, \{0\} \right)^2 
        &= \left\| \frac{1}{T} \sum_{t=1}^{T} f_t \right\|_2^2 \\
        &= \left\| \frac{T-1}{T} \cdot \frac{1}{T-1} \sum_{t=1}^{T-1} f_t  + \frac{1}{T} \cdot f_T  \right\|_2^2 \\
        &= \left( \frac{T-1}{T} \right)^2 d\left( \frac{1}{T-1} \sum_{t=1}^{T-1} f_t, \{0\} \right)^2 + \frac{\| f_T \|_2^2}{T^2} + \frac{2(T-1)}{T^2} \left\langle f_t, \frac{1}{T-1} \sum_{t=1}^{T-1} f_t \right\rangle \\
        &\leq \left( \frac{T-1}{T} \right)^2 d\left( \frac{1}{T-1} \sum_{t=1}^{T-1} f_t, \{0\} \right)^2 + \frac{\| f_T \|_2^2}{T^2} + \frac{2(T-1)}{T^2} \cdot \frac{\| f_T \|_2 \| \theta_{T} \|_2}{\eta(T-1)} \\
        &\leq \left( \frac{T-1}{T} \right)^2 d\left( \frac{1}{T-1} \sum_{t=1}^{T-1} f_t, \{0\} \right)^2 + \frac{L^2}{T^2} + \frac{2Lh}{\eta T^2} 
        \leq  \frac{L^2}{T} + \frac{2Lh}{\eta T}.
    \end{align*}
    The fourth line follows from Cauchy-Schwarz, the choice of $\theta_T = \eta (T-1) \cdot u_T$, and the fact that Blackwell's algorithm returns $u_T = \frac{1}{T-1}\sum_{t=1}^{T-1} f_t$.
    The final line follows from recursion over $t \in [T]$.
    Taking a square root of both sides recovers the OGD error bound from Proposition~\ref{prop:ogd}.
    
    There are two sources of improvement compared to the bound in Equation~\ref{eq:naive}. 
    First, the fact that any projection onto the GEQ target set is 0 means Blackwell's algorithm obtains error $LT^{-1/2}$ instead of $2LT^{-1/2}$, as in Proposition~\ref{prop:proj}.
    Second, the choice of $\alpha(t)$ and the modification to Line 3 allows us to analyze the impact of $h$ differently than in the proof of Theorem~\ref{thm:ba-to-geq}.
    Recall that in that proof, we bounded the number of rounds on which restorativity fails to hold by a constant $m = \lceil \alpha^{-1}(h) \rceil$; normalizing $u_t$ to unit scale is what enables this bounding trick. 
    In the above analysis, we modify Line 3 in a way that allows restorativity to fail on all $T$ rounds, but observe that this error contributes a term of scale $\alpha(t)^{-1}$ to the squared potential analysis of Blackwell's algorithm. 
    Our choice of $\alpha(t) \propto t-1$ then controls the sum of these contributions over $T$ rounds to scale as $T^{-1/2}$.
    The modification in how $u_t$ is normalized in Line 3 is essentially immaterial: we could have had the unnormalized form of Line 3 in Algorithm~\ref{alg:ba-to-geq} if we adopted a slightly more involved definition of a robust BA oracle, but chose to simplify our exposition instead. 
\end{remark}

\section{Reducing Blackwell Approachability to GEQ}\label{sec:geq-to-ba}
In this section, we present a black-box oracle reduction from Blackwell approachability to GEQ with $\Theta = \mathbb R^d$. 
This gives a constructive scheme to convert any GEQ algorithm into a BA algorithm. 

At a high level, our approach is to solve a BA problem $(\mathbb R^d, A, B, f, S)$ that satisfies Assumption~\ref{as:ba} by iteratively querying the GEQ oracle for approach directions, and then translating the approach direction into an action using a valid halfspace oracle $\mathcal O_H$. 
To simplify matters, we begin with two assumptions: first, the target set $S$ is a cone, and second, we have access to an oracle for solving GEQ problems with a constrained decision set $\Theta \subseteq \mathbb R^d$.
The first assumption can be relaxed using \citet{abernethy_blackwell_2011}'s conic lifting argument. 
To relax the second assumption, we present a black-box oracle reduction from GEQ with $\Theta \subseteq \mathbb R^d$ to GEQ with $\Theta = \mathbb R^d$. 
Combining these steps shows that any BA problem that satisfies Assumption~\ref{as:ba} can be solved using a black-box oracle for GEQ problems with $\Theta = \mathbb R^d$. 

\subsection{Warmup: Reducing BA to GEQ with $\Theta \subseteq \mathbb R^d$}
Let $(\mathbb R^d, A, B, f, S)$ be a BA problem that satisfies Assumption~\ref{as:ba} and let $\mathcal O_H$ be a valid halfspace oracle for this problem.
Further suppose that the target set $S$ is a cone; that is, the set of relevant approach directions is exactly the polar $S^\circ$.
For the following result, we assume we have access to an oracle for a GEQ problem of the form $(\mathbb R^d, \Theta, \mathcal G)$ that satisfies Assumption~\ref{as:restor}. 
That is, the oracle returns decisions in a prespecified closed convex decision set that contains the origin, and obtains GEQ in the sense of Definition~\ref{def:geq}, which involves an existential statement about a sequence of normal vectors.
For clarity, we denote this oracle with $\mathbf{GEQ}(\cdot \mid \Theta)$.

We consider the GEQ problem $(\mathbb R^d, \Theta, \mathcal G)$, where
\begin{equation}\label{eq:geq-to-ba-setup}
    \Theta = S^\circ, \quad \mathcal G = \{ g_b: \Theta \to \mathbb R^d \mid \, \text{for all} \,\, \theta \in \Theta , \, g_b(\theta) = -f\left( \mathcal O_H(\theta), b \right) \, \text{for some} \, b \in B\}.
\end{equation}
At each round $t \geq 1$ of the BA problem, we receive approach direction $\theta_t = \mathbf{GEQ}(g_1, \dots, g_{t-1} \mid \Theta)$, play action $a_t = \mathcal O_H(\theta_t)$, observe vector payoff $f(a_t, b_t)$, and construct the vector $g_t = -f(a_t, b_t)$ to feed to $\mathbf{GEQ}(\cdot \mid \Theta)$. 
This reduction is formalized in Algorithm~\ref{alg:geq-to-ba}. 
Below, we show that the constructed GEQ problem $(\mathbb R^d, \Theta, \mathcal G)$ satisfies Assumption~\ref{as:restor}, then analyze the guarantee of our reduction. 

\begin{algorithm}[t]
\caption{BA with GEQ Oracle, $\Theta \subseteq \mathbb R^d$} \label{alg:geq-to-ba}
    \begin{algorithmic}[1]
        \REQUIRE Halfspace oracle $\mathcal O_H$ and GEQ oracle $\mathbf{GEQ}(\cdot \mid S^\circ)$
        \STATE Initialize $g_0 = 0$
        \FOR{$t = 1,2,\dots$}
        \STATE Set approach direction $u_t = \mathbf{GEQ}(g_0, \dots, g_{t-1} \mid S^\circ)$
        \STATE Play action $a_t = \mathcal O_H(u_t)$
        \STATE Observe payoff $f(a_t, b_t)$
        \STATE Let $g_t = -f(a_t, b_t)$
        \ENDFOR
    \end{algorithmic}
\end{algorithm}

\begin{lemma}\label{lem:geq-to-ba-helper}
    Let $(\mathbb R^d, A, B, f, S)$ be a BA problem that satisfies Assumption~\ref{as:ba}, and further suppose that $S$ is a cone. 
    Then the GEQ problem $(\mathbb R^d, \Theta, \mathcal G)$ defined as in Equation~\ref{eq:geq-to-ba-setup} satisfies Assumption~\ref{as:restor}.
\end{lemma}

\begin{proof}[Proof of Lemma~\ref{lem:geq-to-ba-helper}]
    Since $f$ is $L$-bounded, it immediately follows that each $g \in \mathcal G$ is $L$-bounded. 
    Since $S$ is a nonempty closed convex cone, $S^\circ$ is a nonempty closed convex cone, hence contains the origin.
    Finally, for each $g_b \in \mathcal G$, we have that for all $u \in S^\circ$
    \[
    \langle u, -g_b(u) \rangle 
    = \langle u, f\left( \mathcal O_H(u), b \right) \rangle
    \leq \max_{b' \in B} \langle u, f\left( \mathcal O_H(u), b' \right) \rangle 
    \leq \sigma_S(u) 
    \leq 0,
    \]
    where $\mathcal O_H$ is a valid halfspace oracle that maps $u$ to an action $a$ such that $\max_{b' \in B} \langle u, f(a, b') \rangle \leq \sigma_S(u)$. 
    The second inequality holds by the halfspace oracle guarantee.
    The third inequality holds because $u \in S^\circ$. 
    Hence, $\mathcal G$ is $(0, 0)$-restorative. 
\end{proof}

\begin{lemma}\label{lem:geq-to-ba}
    Let $(\mathbb R^d, A, B, f, S)$ be a BA problem that satisfies Assumption~\ref{as:ba} and further suppose that $S$ is a cone. 
    Define the GEQ problem $(\mathbb R^d, \Theta, \mathcal G)$ as in Equation~\ref{eq:geq-to-ba-setup}.
    Fix any $T \geq 1$.
    Then selecting a sequence of approach directions $\theta_1, \dots, \theta_T \in S^\circ$ and a corresponding sequence of actions $a_1, \dots, a_T \in A$ using Algorithm~\ref{alg:geq-to-ba} has the following guarantee. 
    For any sequence of Nature's actions $b_1, \dots, b_T \in B$, there exists a sequence of normal vectors $n_1, \dots, n_T$, with $n_t \in N_{S^\circ}(\theta_t)$ for all $t \in [T]$, such that
    \[
    d\left( \frac{1}{T} \sum_{t=1}^{T} f(a_t, b_t), S \right) 
    \leq \left\| \frac{1}{T} \sum_{t=1}^{T} g_{t}(\theta_t)+n_t \right\|_2
    \leq \frac{\mathbf{GEQ}\mathrm{Err}(L, 0, T)}{T}.
    \]
\end{lemma}

\begin{proof}[Proof of Lemma~\ref{lem:geq-to-ba}]
    By Lemma~\ref{lem:geq-to-ba-helper}, $\mathcal G$ is $L$-bounded and $(0, 0)$-restorative, so the GEQ oracle guarantees the existence of a sequence of normal vectors $n_1, \dots, n_T$, with $n_t \in N_{S^\circ}(\theta_t)$ for all $t \in [T]$, such that
    \[
    \left\| \frac{1}{T} \sum_{t=1}^{T} g_t(\theta_t) + \frac{1}{T} \sum_{t=1}^{T} n_t \right\|_2
    \leq \frac{\mathbf{GEQ}\mathrm{Err}(L, 0, T)}{T}.
    \]
    Next, we claim that $n_t \in S$ for all $t \in [T]$. 
    This is because $N_{S^\circ}(\theta_t) \subseteq S$ for all $t \in [T]$, as is verified in Lemma~\ref{lem:geq-to-ba-helper2} below. 
    Hence, we have that
    \begin{align*}
        \left\| \frac{1}{T} \sum_{t=1}^{T} g_t(\theta_t) + \frac{1}{T} \sum_{t=1}^{T} n_t \right\|_2
        &= \left\| \frac{1}{T} \sum_{t=1}^{T} -f(a_t, b_t) + \frac{1}{T} \sum_{t=1}^{T} n_t \right\|_2 \\
        &= \left\| \frac{1}{T} \sum_{t=1}^{T} f(a_t, b_t) - \frac{1}{T} \sum_{t=1}^{T} n_t \right\|_2 \\
        &\geq \inf_{s \in S} \left\| \frac{1}{T} \sum_{t=1}^{T} f(a_t, b_t) - s \right\|_2 
        = d\left( \frac{1}{T} \sum_{t=1}^{T} f(a_t, b_t), S \right).
    \end{align*}
    The first line holds by definitions. 
    The third line holds because the convexity of $S$ implies $\frac{1}{T} \sum_{t=1}^{T} n_t \in S$.
\end{proof}

\begin{lemma}\label{lem:geq-to-ba-helper2}
    Suppose $S$ is a nonempty closed convex cone and $S^\circ$ is its polar. 
    Then for all $z \in S^\circ$, $N_{S^\circ}(z) \subseteq S$.
\end{lemma}

\begin{proof}[Proof of Lemma~\ref{lem:geq-to-ba-helper2}]
    Fix any $z \in S^\circ$.
    For any $s \in N_{S^\circ}(z)$, $\langle s, z' \rangle \leq \langle s, z \rangle$ for all $z' \in S^\circ$.
    Since $S^\circ$ is a cone, $0 \in S^\circ$ and $2z \in S^\circ$. 
    Instantiating the last inequality with $z' = 0$ and $z' = 2z$ shows that $\langle s, z \rangle = 0$.
    Hence, $\langle s, z' \rangle \leq 0$ for all $z' \in S^\circ$, which shows $s \in (S^\circ)^\circ$.
    Since $S$ is closed and convex, $(S^\circ)^\circ = S$.
\end{proof}

\begin{remark}\label{rmk:indist}
    The normal vectors whose time-average cancels out the time-average of realized vectors---whose existence the GEQ oracle guarantees---play an important role in the proof of Lemma~\ref{lem:geq-to-ba}.
    Namely, the time-average of these normal vectors serve as the primal witness of a vanishing upper bound on the distance between the time-average of vector payoffs and the target set $S$. 
    
    In fact, the proof of Lemma~\ref{lem:geq-to-ba} can be interpreted as a sort of indistinguishability argument. 
    The GEQ oracle selects a sequence of decisions $u_1, \dots, u_T$ with the guarantee that there exists a synthetic sequence of payoffs $n_1, \dots, n_T$ whose time-average $\frac{1}{T}\sum_{t=1}^{T} n_t$ has the following two properties. 
    First, it is contained in the target set. 
    Second, it is ``indistinguishable" from the time-average of the true payoffs, $\frac{1}{T} \sum_{t=1}^{T} f_t$, in the sense that $\left\| \frac{1}{T} \sum_{t=1}^{T} f_t - n_t \right\|$ is guaranteed to vanish as $T$ grows. 
    Hence, it must hold that the time-average of true payoffs approaches the target set as $T \to \infty$.
\end{remark}

\begin{remark}\label{rmk:primal}
    Interestingly, the above remark suggests that our reduction of conic Blackwell approachability to GEQ can be interpreted as the ``primal" counterpart to \citet{abernethy_blackwell_2011}'s reduction of conic approachability to regret minimization, which is naturally interpreted as a ``dual" formulation.
    
    To see the contrast, we briefly review \citet{abernethy_blackwell_2011}'s reduction, which is detailed in Algorithm~\ref{alg:nr-to-ba} and Proposition~\ref{prop:nr-to-ba}.
    At each round $t \geq 1$, this reduction queries a RM oracle with decision set $S^\circ \cap B_2(1)$ to select an approach direction $\theta_t$, maps this direction to an action $a_t$ using a valid halfspace oracle, then constructs the loss vector $\ell_t = -f(a_t, b_t)$ using the realized payoff.
    \citet{abernethy_blackwell_2011}'s analysis begins with the dual formulation of distance to a cone and observes that the best-in-hindsight competitor, $\arg\min_{\theta \in S^\circ \cap B_2(1)} \sum_{t=1}^{T} \langle \ell_t, \theta \rangle = \arg\max_{\theta \in S^\circ \cap B_2(1)} \sum_{t=1}^{T} -\langle \ell_t, \theta \rangle, $ is the dual witness of distance to the target set.
    That is, 
    \begin{align*}
        d\left( \frac{1}{T} \sum_{t=1}^{T} f(a_t, b_t), S \right)
        &= \max_{\theta \in S^\circ \cap B_2(1)} \left\langle \frac{1}{T} \sum_{t=1}^{T} f(a_t, b_t), \theta \right\rangle \\
        &= \max_{\theta \in S^\circ \cap B_2(1)} \frac{1}{T} \sum_{t=1}^{T} -\langle \ell_t, \theta \rangle \\
        &\leq \max_{\theta \in S^\circ \cap B_2(1)} \frac{1}{T} \sum_{t=1}^{T} \langle \ell_t, \theta_t - \theta \rangle,
    \end{align*}
    where the final line holds by the halfspace oracle guarantee: for all $t \in [T]$, $\langle \ell_t, \theta_t \rangle = \langle -f(a_t, b_t), \theta_t \rangle \geq 0$.
    Notably, this reduction hinges on the RM oracle's ability to suppress all dual witnesses of distance, or separating directions, between the time-averaged payoff and the target cone.  
    
    At first glance, our reduction appears to be similar to the \citet{abernethy_blackwell_2011} reduction: we also use the GEQ oracle to select approach directions at each timestep. 
    However, our reduction analyzes the primal formulation of distance, $d(x, S) = \inf_{s \in S}\| x - s \|_2$, and hinges on the GEQ oracle's ability to ensure that a primal witness of vanishing distance exists.
    Hence, the intuition for why the GEQ oracle succeeds is distinct from the intuition for why the RM oracle succeeds.

    In a sense, GEQ and RM algorithms can be interpreted as solving conic approachability problems from the opposite sides of convex duality. 
    Coincidentally, by an argument of \citet{abernethy_blackwell_2011}, any approachability problem can be represented as a conic approachability problem after dimensional lifting. 
    Putting these observations together supplies some intuition for why GEQ and RM algorithms are equally powerful, despite having incomparable guarantees.
\end{remark}

While Lemma~\ref{lem:geq-to-ba} reveals interesting connections between GEQ, approachability, and regret minimization, it does not complete the equivalence we seek.
This is because we assumed access to an oracle for GEQ problems with constrained decision sets, while Theorem~\ref{thm:ba-to-geq} only shows that GEQ problems with unconstrained decision sets can be solved using BA algorithms. 
Next, we show that any GEQ problem with constraints can be solved using an oracle for GEQ problems without constraints, ``closing the loop" on the equivalence.

\subsection{Technical Lemma: Reducing GEQ with $\Theta \subseteq \mathbb R^d$ to GEQ with $\Theta = \mathbb R^d$}
Now we show how to use a black-box oracle for unconstrained GEQ, denoted by $\mathbf{GEQ}(\cdot \mid \mathbb R^d)$, to solve a GEQ problem with constraints. 
Our approach is to transform the constrained problem into an unconstrained problem with a modified class of vector fields. 
In this reduction, we treat Euclidean projection onto a nonempty closed convex set as an elementary operation.

Let $(\mathbb R^d, \Theta, \mathcal G)$ be a GEQ problem that satisfies Assumption~\ref{as:restor}.
Define the GEQ problem $(\mathbb R^d, \mathbb R^d, \bar{\mathcal G})$, where
\begin{equation}\label{eq:geq-to-geq}
    \bar{\mathcal G} = \left\{ \bar g_g: \mathbb R^d \to \mathbb R^d \mid \, \text{ for all } x \in \mathbb R^d, \, \bar g_g(x) = g\left( \Pi_\Theta(x)\right) + n_g(x) \, \text{for some} \, g \in \mathcal G \right\}
\end{equation}
and the function $n_g: \mathbb R^d \to \mathbb R^d$ is defined as
\begin{align*}
    n_g(x) =
    \begin{cases}
        2\left\| g\left( \Pi_\Theta(x) \right) \right\|_2 \cdot \frac{x - \Pi_\Theta(x)}{\| x - \Pi_\Theta(x) \|_2} &\quad \text{if} \quad x \neq \Pi_\Theta(x) \\
        0 &\quad \text{otherwise} 
    \end{cases}
\end{align*}
That is, we introduce an appropriately scaled projection residual of $x$ directly into the vector field.
This modification has two benefits. 
First, the scale of the projection residual ensures that the resulting vector field is restorative. 
Second, each scaled projection residual of the form $c(x-\Pi_\Theta(x))$ belongs to the normal cone $N_\Theta(\Pi_\Theta(x))$.
Hence, achieving the unconstrained notion of GEQ on the vector fields $\bar g(x) = g(\Pi_\Theta(x)) + c(x-\Pi_\Theta(x))$ implies that the constrained notion of GEQ has been achieved on the original vector fields $g(\Pi_\Theta(x))$.

We begin by verifying that $(\mathbb R^d, \mathbb R^d, \bar{\mathcal G})$ satisfies Assumption~\ref{as:restor} with slightly larger constants than the original problem.

\begin{lemma}\label{lem:geq-to-geq-helper}
    Let $(\mathbb R^d, \Theta, \mathcal G)$ be a GEQ problem that satisfies Assumption~\ref{as:restor} with constants $L \geq 0$ and $h \geq 0$. 
    Then the GEQ problem $(\mathbb R^d, \mathbb R^d, \bar{\mathcal G})$ defined in Equation~\ref{eq:geq-to-geq} satisfies Assumption~\ref{as:restor} with constants $3L$ and $2h$. 
\end{lemma}

\begin{proof}[Proof of Lemma~\ref{lem:geq-to-geq-helper}]
    Since every $g \in \mathcal G$ is $L$-bounded, we have that for all $\bar g \in \bar{\mathcal G}$ and $x \in \mathbb R^d$,
    \[
    \| \bar g(x) \|_2 \leq \| g(\Pi_\Theta(x)) \|_2 + 2\| g(\Pi_\Theta(x))\|_2 \leq 3L.
    \]
    It remains to be shown that $\bar{\mathcal G}$ is $(2h, 0)$-restorative. 
    Observe that
    \begin{align*}
        \langle -\bar g(x), x \rangle 
        &= \langle -g(\Pi_\Theta(x)) - n_g(x), \Pi_\Theta(x)+ (x-\Pi_\Theta(x)) \rangle \\
        &= \langle -g(\Pi_\Theta(x)), \Pi_\Theta(x) \rangle + \langle -n_g(x), \Pi_\Theta(x) \rangle +\langle -g(\Pi_\Theta(x)), x - \Pi_\Theta(x) \rangle + \langle -n_g(x), x - \Pi_\Theta(x) \rangle \\
        &\leq \langle -g(\Pi_\Theta(x)), \Pi_\Theta(x) \rangle + \langle -g(\Pi_\Theta(x)), x - \Pi_\Theta(x) \rangle + \langle -n_g(x), x - \Pi_\Theta(x) \rangle,
    \end{align*}
    where the inequality holds because the first-order condition for optimality for Euclidean projection, together with the assumption that $0 \in \Theta$, implies
    \[
    \langle -n_g(x), \Pi_\Theta(x) \rangle \leq 0.
    \]
    There are two cases. 
    First, suppose $x = \Pi_\Theta(x)$, which means $n_g(x) = 0$. 
    Then if $\|x\| \geq 2h$, we have that $\| \Pi_\Theta(x) \| \geq h$, so by the $(h, 0)$-restorativity of $g$, 
    \[
    \langle -\bar g(x), x \rangle 
    \leq \langle -g(\Pi_\Theta(x)), \Pi_\Theta(x) \rangle \leq 0.
    \]
    Second, suppose $x \neq \Pi_\Theta(x)$. 
    Then
    \[
    \langle -\bar g(x), x \rangle 
    \leq \langle -g(\Pi_\Theta(x)), \Pi_\Theta(x) \rangle + \langle -g(\Pi_\Theta(x)), x - \Pi_\Theta(x) \rangle - 2 \| g(\Pi_\Theta(x)) \|_2 \| x-\Pi_\Theta(x) \|_2
    \]
    There are two further sub-cases.
    First, if $\|\Pi_\Theta(x)\|_2 \geq h$, then by the $(h, 0)$-restorativity of $g$,
    \[
    \langle -g(\Pi_\Theta(x)), \Pi_\Theta(x) \rangle \leq 0.
    \]
    Additionally, by Cauchy-Schwarz,
    \begin{align*}
        &\langle -g(\Pi_\Theta(x)), x-\Pi_\Theta(x) \rangle - 2\| g(\Pi_\Theta(x)) \|_2 \| x-\Pi_\Theta(x)\|_2 \\
        &\quad\quad\quad\quad\quad \leq \|g(\Pi_\Theta(x))\|_2 \|x-\Pi_\Theta(x) \|_2 - 2\| g(\Pi_\Theta(x)) \|_2 \| x-\Pi_\Theta(x)\|_2 \\
        &\quad\quad\quad\quad\quad \leq 0.
    \end{align*}
    So the desired inequality holds. 
    Second, if $\| x\| \geq 2h$ but $\| \Pi_\Theta(x) \| < h$, then
    \[
    \| x-\Pi_\Theta(x) \|_2 \geq \| x \|_2 - \|\Pi_\Theta(x)\|_2 > 2h - h > \| \Pi_\Theta(x) \|_2.
    \] 
    Hence, by Cauchy-Schwarz,
    \begin{align*}
        \langle -\bar g(x), x \rangle 
        &\leq \| g(\Pi_\Theta(x))\|_2 \| \Pi_\Theta(x) \|_2 + \| g(\Pi_\Theta(x)) \|_2 \|x-\Pi_\Theta(x) \|_2 - 2\| g(\Pi_\Theta(x)) \|_2 \| x-\Pi_\Theta(x) \|_2 \\
        &\leq \| g(\Pi_\Theta(x)) \|_2 \left( \|\Pi_\Theta(x) \|_2 - \| x-\Pi_\Theta(x) \|_2 \right) \\
        &\leq 0.
    \end{align*}
    So again, the desired inequality holds.
\end{proof}

Our reduction is formalized in Algorithm~\ref{alg:geq-to-geq}. 
At each round $t \geq 1$, given a (possibly) infeasible decision $x_t \in \mathbb R^d$ from the unconstrained GEQ oracle, we play the feasible decision $\theta_t = \Pi_\Theta(x_t)$. 
This approach has the following guarantee.

\begin{algorithm}[t]
\caption{GEQ, $\Theta \subseteq \mathbb R^d$, with GEQ Oracle, $\Theta = \mathbb R^d$} \label{alg:geq-to-geq}
    \begin{algorithmic}[1]
        \REQUIRE Decision set $\Theta$ and GEQ oracle $\mathbf{GEQ}(\cdot \mid \mathbb R^d)$
        \STATE Initialize $\bar g_0 = 0$
        \FOR{$t = 1,2,\dots$}
        \STATE Receive $x_t = \mathbf{GEQ}(\bar g_0, \dots, \bar g_{t-1} \mid \mathbb R^d)$
        \STATE Play $\theta_t = \Pi_\Theta(x_t)$
        \STATE Observe $g_t(\theta_t)$
        \STATE Set $n_t = 
        \begin{cases}
            2 \|g_t(\theta_t)\|_2 \cdot \frac{x_t-\theta_t}{\| x_t-\theta_t \|_2} &\quad \text{if} \quad x_t \neq \theta_t \\
            0 &\quad \text{otherwise}
        \end{cases}$
        \STATE Set $\bar g_t = g_t(\theta_t) + n_t$
        \ENDFOR
    \end{algorithmic}
\end{algorithm}

\begin{lemma}\label{lem:geq-to-geq}
    Let $(\mathbb R^d, \Theta, \mathcal G)$ be a GEQ problem that satisfies Assumption~\ref{as:restor} with constants $L \geq 0$ and $h \geq 0$.
    Define the GEQ problem $(\mathbb R^d, \mathbb R^d, \bar{\mathcal G})$ as in Equation~\ref{eq:geq-to-geq}.
    Fix any $T \geq 1$. 
    Selecting a sequence of (possibly) infeasible decisions $x_1, \dots, x_T \in \mathbb R^d$ and the corresponding sequence of feasible decisions $\theta_1, \dots, \theta_T \in \Theta$ using Algorithm~\ref{alg:geq-to-geq} has the following guarantee.
    For every sequence of vector fields $g_1, \dots, g_T \in \mathcal G$, there exists a sequence of normal vectors $n_1, \dots, n_T$, with $n_t \in N_\Theta(\theta_t)$ for all $t \in [T]$, such that
    \[
    \left\| \frac{1}{T}\sum_{t=1}^{T} g_t(\theta_t) + n_t \right\|_2 
    = \left\| \frac{1}{T} \sum_{t=1}^{T} \bar g_t(x_t) \right\|_2
    \leq \frac{\mathbf{GEQ}\mathrm{Err}(3L, 2h, T)}{T}.
    \]
\end{lemma}

\begin{proof}[Proof of Lemma~\ref{lem:geq-to-geq}]
    By Lemma~\ref{lem:geq-to-geq-helper}, we have that $(\mathbb R^d, \mathbb R^d, \bar{\mathcal G})$ satisfies Assumption~\ref{as:restor} with constants $3L$ and $2h$.
    Hence, $\mathbf{GEQ}(\cdot \mid \mathbb R^d)$ selects a sequence of decisions $x_1, \dots, x_T \in \mathbb R^d$ that enforces 
    \[
    \left\| \frac{1}{T} \sum_{t=1}^{T} \bar g_t(x_t) \right\|_2 \leq \frac{\mathbf{GEQ}\mathrm{Err}(3L, 2h, T)}{T}.
    \]
    Now, recall that $\theta_t = \Pi_\Theta(x_t)$ and observe that
    \[
    \left\| \frac{1}{T} \sum_{t=1}^{T} \bar g_t(x_t) \right\|_2 = \left\| \frac{1}{T} \sum_{t=1}^{T} g_t(\theta_t) + n_{g_t}(x_t) \right\|_2,
    \]
    where each $n_{g_t}(x_t) \in N_\Theta(\theta_t)$ by construction. 
    Hence, the sequence $n_{g_1}(x_1), \dots, n_{g_T}(x_T)$ can be taken as the desired sequence of normal vectors $n_1, \dots, n_T$.
\end{proof}

\begin{remark}
    Algorithm~\ref{alg:geq-to-geq} and Lemma~\ref{lem:geq-to-geq} demonstrate that as long as the Learner can solve Euclidean projection problems, the Learner can readily construct the sequence of normal vectors whose existence the definition of GEQ error guarantees. 
    
\end{remark}

\subsection{Main Result: Reducing BA to GEQ with $\Theta = \mathbb R^d$}
Our final reduction follows from combining the constructions of Lemmas~\ref{lem:geq-to-ba} and~\ref{lem:geq-to-geq}.

Given a BA problem $(\mathbb R^d, A, B, f, S)$ that satisfies Assumption~\ref{as:ba}, where $S$ is also a cone, construct a GEQ problem $(\mathbb R^d, \mathbb R^d, \bar{\mathcal G})$, where
\begin{equation}\label{eq:geq-to-ba-setup2}
    \bar{\mathcal G} = \left\{ g_b: \mathbb R^d \to \mathbb R^d \mid \, \text{for all} \, x \in \mathbb R^d, \, g_b(x) = -f(\mathcal O_H(\Pi_{S^\circ}(x)), b) + n_b(x) \, \text{for some} \, b \in B \right\},
\end{equation}
and the function $n_b: \mathbb R^d \to \mathbb R^d$ is defined as
\begin{align*}
    n_b(x) = 
    \begin{cases}
        2\|f(\mathcal O_H(\Pi_{S^\circ}(x)), b)\|_2 \cdot \frac{x-\Pi_{S^\circ}(x)}{\| x-\Pi_{S^\circ}(x) \|_2}   &\quad \text{if} \quad x \neq \Pi_{S^\circ}(x) \\
        0 &\quad \text{otherwise}.
    \end{cases}
\end{align*}
At each round $t \geq 1$ of the BA problem, given a (possibly) infeasible approach direction $x_t \in \mathbb R^d$ from the unconstrained GEQ oracle, construct the feasible approach direction $\theta_t = \Pi_{S^\circ}(x_t) \in S^\circ$, and play the action $a_t = \mathcal O_H(\theta_t)$.
This reduction is formalized in Algorithm~\ref{alg:geq-to-ba2}, and has the following guarantee.

\begin{algorithm}[t]
    \caption{BA with GEQ Oracle, $\Theta = \mathbb R^d$} \label{alg:geq-to-ba2}
    \begin{algorithmic}[1]
        \REQUIRE Set $S^\circ$, halfspace oracle $\mathcal O_H$, and GEQ oracle $\mathbf{GEQ}(\cdot \mid \mathbb R^d)$
        \STATE Initialize $\bar g_0 = 0$
        \FOR{$t = 1,2,\dots$}
        \STATE Receive $x_t = \mathbf{GEQ}(\bar g_0, \dots, \bar g_{t-1} \mid \mathbb R^d)$
        \STATE Set approach direction $\theta_t = \Pi_{S^\circ}(x_t)$
        \STATE Play action $a_t = \mathcal O_H(\theta_t)$
        \STATE Observe payoff $f(a_t, b_t)$
        \STATE Set $n_t = 
        \begin{cases}
            2 \|f(a_t, b_t)\|_2 \cdot \frac{x_t-\theta_t}{\| x_t-\theta_t \|_2} &\quad \text{if} \quad x_t \neq \theta_t \\
            0 &\quad \text{otherwise}
        \end{cases}$
        \STATE Let $\bar g_t = -f(a_t, b_t) + n_t$
        \ENDFOR
    \end{algorithmic}
\end{algorithm}

\begin{theorem}\label{thm:geq-to-ba}
    Let $(\mathbb R^d, A, B, f, S)$ be a BA problem that satisfies Assumption~\ref{as:ba} and further suppose that $S$ is a cone. 
    Define the GEQ problem $(\mathbb R^d, \mathbb R^d, \bar{\mathcal G})$ as in Equation~\ref{eq:geq-to-ba-setup2}. 
    Fix $T \geq 1$.
    Then selecting a sequence of directions $x_1, \dots, x_T \in \mathbb R^d$ and a corresponding sequence of actions $a_1, \dots, a_T \in A$ using Algorithm~\ref{alg:geq-to-ba2} has the following guarantee.
    For any sequence of Nature's actions $b_1, \dots, b_T \in B$,
    \[
    d\left( \frac{1}{T} \sum_{t=1}^{T} f(a_t, b_t), S\right) 
    \leq \left\| \frac{1}{T} \sum_{t=1}^{T} \bar g_t(x_t) \right\|_2
    \leq \frac{\mathbf{GEQ}\mathrm{Err}(3L, 0, T)}{T}.
    \]
\end{theorem}

\begin{proof}[Proof of Theorem~\ref{thm:geq-to-ba}]
    Define 
    \[
    \mathcal G = \left\{ -f(\mathcal O_H(\cdot), b): S^\circ \to \mathbb R^d \mid b \in B  \right\}.
    \]
    By Lemma~\ref{lem:geq-to-ba-helper}, the GEQ problem $(\mathbb R^d, S^\circ, \mathcal G)$ satisfies Assumption~\ref{as:restor} with constants $L$ and $h = 0$. 
    By Lemma~\ref{lem:geq-to-geq-helper}, the GEQ problem $(\mathbb R^d, \mathbb R^d, \bar{\mathcal G})$ satisfies Assumption~\ref{as:restor} with constants $3L$ and $2h = 0$. 
    Hence, 
    \begin{align*}
        d\left( \frac{1}{T} \sum_{t=1}^{T} f(a_t, b_t), S\right) 
        &\leq \left\| \frac{1}{T} \sum_{t=1}^{T} f(\mathcal O_H(\theta_t), b_t) - n_t \right\|_2 \\
        &= \left\| \frac{1}{T} \sum_{t=1}^{T} f(\mathcal O_H(\Pi_{S^\circ}(x_t)), b_t) - 2\left\| f(\mathcal O_H(\Pi_{S^\circ}(x_t)), b_t) \right\|_2 \cdot \frac{x_t - \Pi_{S^\circ}(x_t)}{\| x_t - \Pi_{S^\circ}(x_t) \|_2} \right\|_2 \\
        &= \left\| \frac{1}{T} \sum_{t=1}^{T} \bar g_t(x_t) \right\|_2 
        \leq \frac{\mathbf{GEQ}\mathrm{Err}(3L, 0, T)}{T},
    \end{align*}
    where  $n_t = 2\left\| f(\mathcal O_H(\Pi_{S^\circ}(x_t)), b_t) \right\|_2 \cdot \frac{x_t - \Pi_{S^\circ}(x_t)}{\| x_t - \Pi_{S^\circ}(x_t) \|_2}$.
    The first inequality holds by the fact that $n_t \in N_{S^\circ}(\theta_t)$, Lemma~\ref{lem:geq-to-ba-helper2}, and the convexity of $S$.
    The second and third lines follow from definitions.
    The final inequality holds by Lemma~\ref{lem:geq-to-geq}.
    \end{proof}
    
\begin{remark}
    Suppose $(\mathbb R^d, A, B, f, S)$ is a BA problem that satisfies Assumption~\ref{as:ba} but $S$ is not a cone. 
    Then we can follow \citet{abernethy_blackwell_2011}'s argument (see Proposition~\ref{prop:conic2}) to construct a lifted BA problem $(\mathbb R^{d+1}, A, B, \bar f, \bar S)$, where 
    \[
    \bar f(a, b) = 3L \oplus f(a, b), \quad \bar S = \mathrm{cone}(\{3L\} \times (S \cap B_2(3L))).
    \]
    Then by Proposition~\ref{prop:conic}, solving the lifted BA problem $(\mathbb R^{d+1}, A, B, \bar f, \bar S)$ implies a solution for the original BA problem $(\mathbb R^d, A, B, f, S)$.
    Hence, Theorem~\ref{thm:geq-to-ba} can be invoked on the lifted problem.
\end{remark}

\section{Applications}\label{sec:algo}
In this section, we study algorithmic applications of the equivalence between Blackwell approachability and GEQ. 
Namely, we compose our reductions with \citet{abernethy_blackwell_2011}'s reductions between Blackwell approachability and regret minimization, which demonstrate how RM algorithms can be used to solve GEQ problems, and vice versa.
One direction of this composed reduction yields a principled way to transfer refined guarantees, such as optimism and strong adaptivity, from RM to GEQ. 
The other direction yields a way to solve RM problems with no step size parameter, and with no knowledge of the loss vector norm bound $L$ and time horizon $T$. 

\subsection{Reducing GEQ to Regret Minimization}
Composing our reduction from GEQ with $\Theta = \mathbb{R}^d$ to Blackwell approachability with the known reduction from approachability to regret minimization, discussed in Appendix~\ref{sec:equiv}, yields a reduction from GEQ to regret minimization. 

Concretely, given a GEQ problem $(\mathbb R^d, \mathbb R^d, \mathcal G)$ that satisfies Assumption~\ref{as:restor}, we construct the RM problem $(\mathbb R^d, B_2(1), B_2(L))$, and at each time $t \geq 1$, treat the realized vector $g_t(\theta_t) \in B_2(L)$ as the loss vector. 
The resulting reduction, formalized in Algorithm~\ref{alg:nr-to-geq}, closely resembles Algorithm~\ref{alg:ba-to-geq}, with calls to a BA oracle replaced by calls to a RM oracle. 

\begin{algorithm}[t]
\caption{GEQ, $\Theta = \mathbb{R}^d$, with RM Oracle}
\label{alg:nr-to-geq}
    \begin{algorithmic}[1]
    \REQUIRE Positive and increasing function $\alpha: \mathbb R_{+} \to \mathbb R_{+}$ and RM oracle $\mathbf{RM}$
    \STATE Initialize $g_0 = 0$ and arbitrary $\theta_0 \in \Theta$
    \FOR{$t = 1,2,\dots$}
    \STATE Set approach direction $u_t = \mathbf{RM}(g_0, \dots, g_{t-1})$
    \IF{$u_t \neq 0$}
        \STATE Play action $\theta_t =  \alpha(t) \cdot \frac{u_t}{\|u_t\|_2}$
    \ELSE
        \STATE Play action $\theta_t = \theta_0$
    \ENDIF
    \STATE Observe vector $g_t(\theta_t)$ and set $g_t = g_t(\theta_t)$
    \ENDFOR
  \end{algorithmic}
\end{algorithm}

\begin{corollary}\label{cor:nr-to-geq}
    Let $(\mathbb R^d, \mathbb{R}^d, \mathcal G)$ be a GEQ problem that satisfies Assumption~\ref{as:restor}.
    Fix any $T \geq 1$.
    Then using Algorithm~\ref{alg:nr-to-geq} to select the sequence of actions $\theta_1, \dots, \theta_T \in \mathbb R^d$ guarantees that, for any sequence of vector fields $g_1, \dots, g_T \in \mathcal{G}$, 
    \[
    \left\| \frac{1}{T} \sum_{t=1}^{T} g_t(\theta_t) \right\|_2 
    \leq \frac{\mathbf{RM}\mathrm{Err}(1, L, T)}{T} + \frac{L \lceil \alpha^{-1}(h) \rceil }{T}.
    \]
\end{corollary}

\begin{proof}[Proof of Corollary~\ref{cor:nr-to-geq}]
    From Theorem~\ref{thm:ba-to-geq} and the dual formulation of distance to a cone, we have that
    \[ 
    \left\| \frac{1}{T} \sum_{t=1}^{T} g_t(\theta_t) \right\|_2 
    = d\left( -\frac{1}{T} \sum_{t=1}^{T} g_t(\theta_t), \{0\} \right) 
    = \max_{u \in B_2(1)} \left\langle - \frac{1}{T} \sum_{t=1}^{T} g_t(\theta_t), u \right\rangle.
    \]
    We would like to upper bound the final term with regret,  $\max_{u \in B_2(1)} T^{-1} \sum_{t=1}^{T} \langle g_t(\theta_t), u_t - u\rangle$.
    This requires that $T^{-1} \sum_{t=1}^{T} \langle g_t(\theta_t), u_t \rangle \geq 0$, but on the rounds where restorativity fails,  $\langle g_t(\theta_t), u_t \rangle < 0$. 
    By Theorem~\ref{thm:ba-to-geq}, there are at most $\lceil \alpha^{-1}(h) \rceil$ such rounds, and on each round with a failure, $\langle g_t(\theta_t), u_t \rangle \geq -L$ by Cauchy-Schwarz and the $L$-boundedness of $g_t$.
    Hence, we have that
    \begin{align*}
        \max_{u \in B_2(1)} \left\langle - \frac{1}{T} \sum_{t=1}^{T} g_t(\theta_t), u \right\rangle 
        &\leq \frac{1}{T} \max_{u \in B_2(1)} \sum_{t=1}^{T} \langle g_t(\theta_t), u_t - u \rangle + \frac{L\lceil \alpha^{-1}(h) \rceil}{T} \\
        &\leq \frac{\mathbf{RM}\mathrm{Err}(1, L, T)}{T} + \frac{L\lceil \alpha^{-1}(h) \rceil}{T},
    \end{align*}
    where the last line follows from the RM oracle's guarantee.
\end{proof}

\begin{remark}
    Corollary~\ref{cor:nr-to-geq} prescribes a simple scheme to convert RM algorithms into GEQ algorithms, which can be intuitively interpreted even without invoking the Blackwell approachability framework.  
    Each decision $u_t$ that is returned by an RM algorithm should be scaled up to ensure that there are a bounded number of rounds where restorativity does not hold. 
    This prescription is consistent with the observation that $\sum_{t=1}^{T} \langle g_t(\theta_t), u_t - u \rangle$ upper bounds $\max_{u \in B_2(1)}\left\langle \sum_{t=1}^{T} g_t(\theta_t), u \right\rangle = \| \sum_{t=1}^{T} g_t(\theta_t) \|_2$ as long as the $\sum_{t=1}^{T} \langle g_t(\theta_t), u_t \rangle$ is non-negative.
    Re-scaling decisions is exactly what is needed to enforce this condition on all but a bounded number of rounds.
\end{remark}

Concretely, Algorithm~\ref{alg:nr-to-geq} can serve as a principled way to translate any refined guarantee from RM to GEQ. 
That is, existing GEQ algorithms minimize worst-case GEQ error on average over the entire time horizon, with error bounds that scale like $O(LT^{-1/2})$. 
In contrast, there exist countless refined notions of error in RM that are small when the problem instance is ``easier" in some sense.

For example, there are optimistic algorithms that, at each round $t \geq 1$, construct a prediction $m_t$ of the loss vector $g_t$, and use it to choose action $\theta_t$ before observing the true $g_t$. 
The resulting optimistic regret bounds scale with prediction error, rather than in terms of $L$ and $T$. 
Hence, these bounds are small when loss vectors are predictable. 
By plugging such an optimistic algorithm into our reduction, we derive a GEQ algorithm that obtains optimistic error bounds. 

\begin{corollary}\label{cor:optim}
    Let $(\mathbb R^d, \mathbb{R}^d, \mathcal G)$ be a GEQ problem that satisfies Assumption~\ref{as:restor}.
    Fix any $T \geq 1$.
    Instantiate $\mathbf{RM}$ in Algorithm~\ref{alg:nr-to-geq} with the optimistic algorithm from Proposition~\ref{prop:optim}.
    Then using Algorithm~\ref{alg:nr-to-geq} to select the sequence of actions $\theta_1, \dots, \theta_T \in \mathbb R^d$ guarantees that, for any sequence of vector fields $g_1, \dots, g_T \in \mathcal{G}$, 
    \[
    \left\| \frac{1}{T} \sum_{t=1}^{T} g_t(\theta_t) \right\|_2 
    \leq O\left( \frac{\sqrt{\sum_{t=1}^{T} \| g_t(\theta_t) - m_t \|^2_2}}{T} \right) + \frac{L \lceil \alpha^{-1}(h) \rceil}{T}.
    \]
\end{corollary}
The proof is immediate from combining Corollary~\ref{cor:nr-to-geq} with Proposition~\ref{prop:optim}, and is omitted.

In another direction, by plugging into our reduction known RM algorithms with strongly adaptive guarantees that hold over all sub-intervals of time in $[0, T]$ at a rate that adapts to sub-interval length, we derive a GEQ algorithm that obtains strongly adaptive error bounds. 

\begin{corollary}\label{cor:sar}
    Let $(\mathbb R^d, \mathbb{R}^d, \mathcal G)$ be a GEQ problem that satisfies Assumption~\ref{as:restor}.
    Fix any $T \geq 1$.
    Instantiate $\mathbf{RM}$ in Algorithm~\ref{alg:nr-to-geq} with the strongly adaptive algorithm from Proposition~\ref{prop:sar}.
    Then using Algorithm~\ref{alg:nr-to-geq} to select the sequence of actions $\theta_1, \dots, \theta_T \in \mathbb R^d$ guarantees that, for any sequence of vector fields $g_1, \dots, g_T \in \mathcal{G}$ and every $1 \leq t_1 < t_2 \leq T$, 
    \begin{align}
        \left\| \frac{1}{t_2 - t_1+1} \sum_{t \in [t_1, t_2]} g_t(\theta_t) \right\|_2 \leq O\left( \frac{L(1+\sqrt{\log(t_2)})}{\sqrt{t_2-t_1+1}} + \frac{L \lceil \alpha^{-1}(h) \rceil}{t_2 - t_1+1} \right).
    \end{align}
\end{corollary}
The proof is immediate from combining Corollary~\ref{cor:nr-to-geq} with Proposition~\ref{prop:sar}, and is omitted.

\subsection{Reducing Regret Minimization to GEQ}\label{sec:nr}
Composing the known reduction from regret minimization to Blackwell approachability, discussed in Appendix~\ref{sec:equiv}, with our reduction from approachability to GEQ with $\Theta = \mathbb R^d$ yields a reduction from regret minimization to GEQ with $\Theta = \mathbb R^d$. 

Concretely, given a RM problem $(\mathbb R^d, \Theta, \mathcal G)$, we construct the GEQ problem $\left( \mathbb R^{d+1}, \mathbb R^{d+1}, \bar{\mathcal G} \right)$, where
\[
\bar{\mathcal G} = \left\{ \bar g_g: \mathbb R^{d+1} \to \mathbb R^{d+1} \mid \text{ for all } x \in \mathbb R^{d+1}, \, \bar g_g(x) = -f\left( \theta\left( \Pi_{S^\circ}(x) \right), g \right) + n_g(x) \text{ for some } g \in \mathcal G \right\},
\]
the mapping $n_g(x): \mathbb R^{d+1} \to \mathbb R^{d+1}$ is defined as
\begin{align*}
    n_g(x) =
    \begin{cases}
        2\| f\left( \theta \left( \Pi_{S^\circ}(x) \right), g \right) \|_2 \cdot \frac{x - \Pi_{S^\circ}(x)}{\| x - \Pi_{S^\circ}(x) \|_2} &\quad \text{if} \quad x \neq \Pi_{S^\circ}(x) \\
        0 &\quad \text{otherwise},
    \end{cases}
\end{align*}
and $f(\theta, g)$, $S$, and $\theta(u)$ are defined as in Equations~\ref{eq:rm-setup} and \ref{eq:rm-oh}. 
This reduction is formalized in Algorithm~\ref{alg:geq-to-nr}, which specializes Algorithm~\ref{alg:geq-to-ba2} for the conic approachability representation of RM. 

\begin{algorithm}[t]
\caption{RM with GEQ Oracle, $\Theta = \mathbb R^{d+1}$}
\label{alg:geq-to-nr}
\begin{algorithmic}[1]
\REQUIRE Set $S^\circ$ and GEQ oracle $\mathbf{GEQ}(\cdot \mid \mathbb R^{d+1})$
\STATE Initialize $\bar g_0 = 0$
\FOR{$t = 1,2,\dots$}
    \STATE Receive $x_t = \mathbf{GEQ}(\bar g_0, \dots, \bar g_{t-1} \mid \mathbb R^{d+1})$
    \STATE Set approach direction $u_t = \Pi_{S^\circ}(x_t)$
    \STATE Play action $\theta_t = \theta(u_t)$
    \STATE Observe loss vector $g_t$
    \STATE Set $f_t = \frac{\langle g_t, \theta_t \rangle}{B} \oplus -g_t$
    \STATE Set $n_t = 
    \begin{cases}
        2 \|f_t\|_2 \cdot \frac{x_t-u_t}{\| x_t-u_t \|_2} &\quad \text{if} \quad x_t \neq u_t \\
        0 &\quad \text{otherwise}
    \end{cases}$
    \STATE Let $\bar g_t = -f_t + n_t$
\ENDFOR
\end{algorithmic}
\end{algorithm}

\begin{corollary}\label{cor:geq-to-nr}
    Let $(\mathbb R^d, \Theta, \mathcal{G})$ be a RM problem that satisfies Assumption~\ref{as:rm}.
    Fix any $T \geq 1$.
    Selecting a sequence of actions $\theta_1, \dots, \theta_T \in \Theta$ using Algorithm~\ref{alg:geq-to-nr} guarantees that for any sequence of loss vectors $g_1, \dots, g_T \in \mathcal{G}$,
    \begin{align}
        \max_{\theta \in \Theta} \frac{1}{T} \sum_{t=1}^{T} \langle g_t, \theta_t - \theta \rangle \leq \frac{2B
        \cdot \mathbf{GEQ}\mathrm{Err}(3\sqrt{2} L, 0, T)}{T}.
\end{align}
\end{corollary}

\begin{proof}[Proof of Corollary~\ref{cor:geq-to-nr}]
    The proof is immediate from the following chain of inequalities. 
    \[
    \max_{\theta \in \Theta} \frac{1}{T} \sum_{t=1}^{T} \langle g_t, \theta_t - \theta \rangle 
    \leq 2B \cdot d\left( \frac{1}{T} \sum_{t=1}^{T} f(\theta_t, g_t), S \right) 
    \leq 2B \cdot \left\| \frac{1}{T} \sum_{t=1}^{T} \bar g_t(x_t) \right\| 
    \leq \frac{2B \cdot \mathbf{GEQ}\mathrm{Err}(3\sqrt{2} L, 0, T)}{T}.
    \]
    The first inequality holds by Proposition~\ref{prop:ba-to-nr}.
    The second inequality holds by Theorem~\ref{thm:geq-to-ba}. 
    The third inequality holds by the GEQ oracle guarantee. 
\end{proof}

\begin{remark}
    To make concrete what it means for the GEQ oracle to solve RM through an ``indistinguishability" argument, as alluded to in Remark~\ref{rmk:indist}, we carefully interpret Algorithm~\ref{alg:geq-to-nr}.

    Fix $t \in [T]$ and say $n_t = \lambda_t \oplus h_t$ and $u_t = c_t(B \oplus \theta_t)$, where $c_t \geq 0$. 
    Since each $n_t \in N_{S^\circ}(u_t)$ and each $N_{S^\circ}(u_t) = S \cap u_t^\perp$, we have that $n_t \in S$ and $\langle n_t, u_t \rangle = 0$. 
    The first condition means
    \[
    \langle n_t, B \oplus \theta \rangle = \lambda_t B + \langle h_t, \theta \rangle \leq 0 \quad \forall \, \theta \in \Theta,
    \]
    and the second condition means 
    \[
    \langle n_t, u_t \rangle = \lambda_t B + \langle h_t, \theta_t \rangle = 0.
    \]
    Together, these conditions imply that $\langle h_t, \theta \rangle \leq \langle h_t, \theta_t \rangle$ for all $\theta \in \Theta$.
    That is, the vector $h_t$ encoded in the normal $n_t$ is the negation of the loss that would have led to the decision $\theta_t$ that was played being optimal. 
    The time-average of such normals, $\frac{1}{T} \sum_{t=1}^{T} n_t \in S$, can be interpreted as a synthetic ``transcript" of play that is known to have no regret to all competitor actions. 
    The GEQ oracle guarantee, then, establishes that the true transcript of play, $\frac{1}{T} \sum_{t=1}^{T} f_t$, becomes arbitrarily close to the synthetic transcript of play as $T$ grows.
    Formally, 
    \begin{align*}
        \max_{\theta \in \Theta} \frac{1}{T} \sum_{t=1}^{T} \langle g_t, \theta_t - \theta \rangle 
        &= \max_{\theta \in \Theta} \frac{1}{T} \sum_{t=1}^{T} \left\langle f_t, B \oplus \theta \right\rangle \\
        &= \max_{\theta \in \Theta}\left( \left\langle \frac{1}{T} \sum_{t=1}^{T}  n_t, B \oplus \theta \right\rangle + \left\langle \frac{1}{T} \sum_{t=1}^{T}  f_t - \frac{1}{T} \sum_{t=1}^{T}  n_t, B \oplus \theta \right\rangle \right)\\
        &\leq \max_{\theta \in \Theta} \left\| B \oplus \theta \right\|_2 \cdot \left\| \frac{1}{T} \sum_{t=1}^{T}  f_t - \frac{1}{T} \sum_{t=1}^{T}  n_t \right\|_2 \\
        &\leq 2B \cdot \frac{\mathbf{GEQ}\mathrm{Err}(3\sqrt{2} L, 0, T)}{T}.
    \end{align*}
    The third and fourth lines establish that $\frac{1}{T} \sum_{t=1}^{T} f_t$ and $\frac{1}{T} \sum_{t=1}^{T} n_t$ are indistinguishable, where tests range over competitor decisions $\theta \in \Theta$.
\end{remark}

\begin{remark}
    When the GEQ oracle $\mathbf{GEQ}(\cdot \mid \mathbb R^{d+1})$ is instantiated with OGD, the regret bound of Corollary~\ref{cor:geq-to-nr} becomes $O(BL T^{-1/2})$, which is known to be optimal for RM. 
    However, recall that OGD for GEQ can be implemented with any constant step size $\eta = O(1)$ that is chosen with no knowledge of the loss vector norm bound $L$ or the time horizon $T$. 
    Moreover, the choice of constant $\eta$ does not appear in the final regret bound. 
    So without loss of generality, OGD can always be implemented with $\eta = 1$, ensuring that there are no parameters to be tuned in the implementation of Algorithm~\ref{alg:geq-to-nr}.
\end{remark}

\begin{remark}
    A natural question is whether it is possible to represent a $d$-dimensional RM problem as a $d$-dimensional BA problem, rather than a $(d+1)$-dimensional BA problem. 
    Comparing known upper bounds for GEQ and lower bounds for RM reveals that a general dimension-preserving reduction from RM to BA cannot exist.
    \citet{JMLR:v26:25-0356} showed that OGD solves all one-dimensional GEQ problems that satisfy Assumption~\ref{as:restor} at rate $O(T^{-1})$. 
    In contrast, it is well known that the minimax lower bound for one-dimensional RM is $\Omega(T^{-1/2})$.
    Hence, if a general reduction that represents $d$-dimensional RM problems as $d$-dimensional BA problems existed, we could use the one-dimensional specialization of Algorithm~\ref{alg:geq-to-nr} with OGD instantiating the GEQ oracle to solve all one-dimensional RM problems at rate $O(T^{-1})$
\end{remark}

\section{Discussion}\label{sec:disc}
In this work, we showed that Blackwell approachability and GEQ are equivalent in the algorithmic sense: an approachability problem can always be solved using queries to a black-box GEQ oracle, and vice versa. 
Taken together with known results from \citet{abernethy_blackwell_2011}, our reductions imply that GEQ is equivalent to many classical online learning frameworks, such as regret minimization and calibration.
We discussed the connection to regret minimization in detail in Sections~\ref{sec:geq-to-ba} and \ref{sec:algo}.
The connection to calibration also follows from results in \citet{abernethy_blackwell_2011} and \citet{perchet2013approachability}, but we omit details. 

An interesting direction for future work is studying whether non-conic approachability problems can be solved using a GEQ oracle, without first being transformed into a conic approachability form. 
Such a result would be analogous to \citet{shimkin_online_2016}'s result, which showed that non-conic approachability problems can be solved using an online convex optimization oracle, without first being transformed into a conic form. 

\newpage

\acks{
This work was supported in part by the National Science Foundation under grants CCF-2145898, by the Office of Naval Research under grant N00014-24-1-2159, a Google Research Scholar Award, an Alfred P. Sloan fellowship, and a Schmidt Science AI2050 fellowship.
Funded by the European Union (ERC-2022-SYG-OCEAN-101071601). Views and opinions expressed are however those of the author(s) only and do not
necessarily reflect those of the European Union or the European Research Council Executive Agency. Neither the European Union nor the granting authority can be
held responsible for them.
}

\bibliographystyle{plainnat}
\bibliography{geq} 

\appendix

\newpage

\section{Additional Models and Preliminaries}
\subsection{Robust Blackwell Approachability}\label{sec:proj-rob}

\begin{proof}[Proof of Lemma~\ref{lem:proj-rob}]
    Let $\bar f_t = \frac{1}{t}\sum_{s=1}^{t} f_s$. 
Then by the standard analysis of Blackwell's algorithm, 
\begin{align*}
    d\left( \bar f_T, S \right)^2
    &= \left\| \bar f_T - \Pi_S(\bar f_T) \right\|_2^2 \\
    &\leq \left\| \bar f_T -\Pi_S(\bar f_{T-1}) \right\|_2^2 \\
    &= \left\| \frac{T-1}{T} \left( \bar f_{T-1} - \Pi_S(\bar f_{T-1}) \right) + \frac{1}{T} \left( f_T -  \Pi_{S}(\bar f_{T-1}) \right) \right\|_2^2 \\
    &= \left( \frac{T-1}{T} \right)^2 d\left( \bar f_{T-1}, S \right)^2 + \frac{1}{T^2} \left\| f_T - \Pi_S(\bar f_{T-1}) \right\|_2^2 + \frac{2(T-1)}{T^2} \left\langle \bar f_{T-1} - \Pi_S(\bar f_{T-1}), f_T - \Pi_S(\bar f_{T-1}) \right\rangle.
\end{align*}
Next, we depart from the standard analysis in bounding the second and third terms.
Since it holds that
\begin{align*}
    \left\| f_T - \Pi_S(\bar f_{T-1}) \right\|_2^2
    &= \left\| f_T - \bar f_{T-1} \right\|_2^2 - \left\| \bar f_{T-1} - \Pi_S(\bar f_{T-1}) \right\|_2^2 + 2\left\langle f_T - \Pi_S(\bar f_{T-1}), \bar f_{T-1} - \Pi_{S}(\bar f_{T-1}) \right\rangle \\
    &\leq \left\| f_T - \bar f_{T-1} \right\|_2^2 + 2\left\langle f_T - \Pi_S(\bar f_{T-1}), \bar f_{T-1} - \Pi_{S}(\bar f_{T-1}) \right\rangle,
\end{align*}
the final line can be further upper bounded to yield
\begin{align*}
    d\left( \bar f_T, S \right)^2 
    &\leq \left( \frac{T-1}{T} \right)^2 d\left( \bar f_{T-1}, S \right)^2 + \frac{1}{T^2} \left\| f_T -\bar f_{T-1}\right\|_2^2 + \frac{2T}{T^2} \left\langle \bar f_{T-1} - \Pi_S(\bar f_{T-1}), f_T - \Pi_S(\bar f_{T-1}) \right\rangle \\
    &\leq \left( \frac{T-1}{T} \right)^2 d\left( \bar f_{T-1}, S \right)^2 + \frac{4L^2}{T^2} + \frac{2T}{T^2} \left\langle \bar f_{T-1} - \Pi_S(\bar f_{T-1}), f_T - \Pi_S(\bar f_{T-1}) \right\rangle \\
    &\leq \frac{4L^2}{T} + \frac{2}{T^2} \sum_{t=1}^{T} t \langle \bar f_{t-1} - \Pi_S(\bar f_{t-1}), f_t - \Pi_S(\bar f_{t-1}) \rangle,
\end{align*}
where the second line holds because $\| f_T - \bar f_{T-1} \|_2 \leq 2L$ and the final line holds by recursion over $t \in [T]$.

Consider the sum.
If all actions are chosen by a valid halfspace oracle, then each summand is non-positive, hence the sum can be bounded by 0; this is how the standard analysis of Blackwell's algorithm proceeds. 
But by the assumption that an erroneous action is played at most $m$ times, at most $m$ summands are non-negative. 
Moreover, each summand is bounded as follows.
Since $S \cap B_2(L) \neq \emptyset$, fix arbitrary $s_0 \in S \cap B_2(L)$ and observe that for each $t \in [T]$,
\begin{align*}
    \left\langle \bar f_{t-1} - \Pi_S(\bar f_{t-1}), f_t - \Pi_S(\bar f_{t-1}) \right\rangle 
    &= \left\langle \bar f_{t-1} - \Pi_S(\bar f_{t-1}), f_t - s_0 \right\rangle + \left\langle \bar f_{t-1} - \Pi_S(\bar f_{t-1}), s_0 - \Pi_S(\bar f_{t-1}) \right\rangle \\
    &\leq \left\langle \bar f_{t-1} - \Pi_S(\bar f_{t-1}), f_t - s_0 \right\rangle \\
    &\leq \left\| \bar f_{t-1} - \Pi_S(\bar f_{t-1}) \right\|_2 \left\| f_t - s_0 \right\|_2 \\
    &\leq 4L^2,
\end{align*}
where the second line holds because the first-order condition for Euclidean projection implies that 
\[
\left\langle \bar f_{t-1} - \Pi_S(\bar f_{t-1}), s_0 - \Pi_S(\bar f_{t-1}) \right\rangle \leq 0,
\]
the third line holds by Cauchy-Schwarz, and the final line holds by the $L$-boundedness of $f$ and $s_0$.
Hence, we conclude that 
\[
d\left( \bar f_T, S \right)^2 \leq \frac{4L^2}{T} + \frac{8L^2m}{T},
\]
which implies that
\[
d\left( \bar f_T, S \right) \leq C_m \cdot \frac{2L}{\sqrt{T}}
\]
with $C_m = \sqrt{1+2m}$.
\end{proof}

\subsection{Regret Minimization}\label{sec:rm}
Given a decision set $\Theta \subseteq \mathbb R^d$, the regret of a sequence of decisions $\{ \theta_t \}_{t \geq 1}$, with $\theta_t \in \Theta$ for all $t \geq 1$, on a sequence of losses $\{ \ell_t \}_{t \geq 1}$ is measured as $\sup_{\theta \in \Theta} \frac{1}{T} \sum_{t=1}^{T} \ell_t(\theta_t) - \ell_t(\theta)$.
A sufficient condition for achieving vanishing regret is for losses to be subdifferentiable,  Lipschitz, and convex, and for the decision set to be convex and compact.
Under these conditions, online convex optimization can be reduced to online linear optimization: 
\[
\frac{1}{T} \sum_{t=1}^{T} \ell_t(\theta_t) - \ell_t(\theta) \leq \frac{1}{T} \sum_{t=1}^{T} \langle g_t(\theta_t), \theta_t - \theta \rangle,
\]
where $g_t(\theta) \in \partial \ell_t(\theta)$. 
The resulting expression can be interpreted as lifting the variational inequality for first-order optimality in offline convex optimization to the online setting. 
Hence, we use the terms ``regret minimization" (RM) and ``online linear optimization" interchangeably. 

\begin{definition}[Regret minimization]
    A RM problem is a tuple $(\mathbb R^d, \Theta, \mathcal{G})$, where $\Theta \subseteq \mathbb{R}^d$ and $\mathcal{G} \subseteq \mathbb{R}^d$. 
\end{definition}

\begin{assumption}[Online linear optimization]\label{as:rm}
    A RM problem $(\mathbb R^d, \Theta, \mathcal G)$ satisfies the following conditions: $\Theta$ is convex, compact, and $\sup_{\theta \in \Theta} \| \theta \|_2 \leq B$ for some $B \geq 0$, and $\sup_{g \in \mathcal{G}} \|g\|_2 \leq L$ for some $L \geq 0$.
\end{assumption}

At each time $t \geq 1$, the Learner selects a decision $\theta_t \in \Theta$, Nature selects a loss vector $g_t \in \mathcal G$, then the Learner observes loss $\langle g_t, \theta_t \rangle$ and gradient $g_t$.
The Learner's goal is formalized as follows. 

\begin{definition}[Minimizing regret]
    Let $(\mathbb R^d, \Theta, \mathcal{G})$ be an OLO problem.
    A sequence of decisions $\{ \theta_t \}_{t \geq 1}$, with $\theta_t \in \Theta$ for all $t \geq 1$, is said to minimize regret on a sequence of loss vectors $\{ g_t \}_{t \geq 1}$, with $g_t \in \mathcal{G}$ for all $t \geq 1$, if 
    \[
    \sup_{\theta \in \Theta} \frac{1}{T}\sum_{t=1}^{T} \langle g_t, \theta_t - \theta \rangle \to 0 \quad \text{as} \quad T \to \infty.
    \]
\end{definition}

Regret minimization algorithms are well-studied in the literature \citep{cesa-bianchi_prediction_2006, orabona_modern_2025}, and we define an abstract regret minimization oracle as follows. 
\begin{definition}[RM oracle]
    A RM oracle, denoted by $\mathbf{RM}$, is a family of strategies indexed by valid RM problems. 
    Let $(\mathbb R^d, \Theta, \mathcal G)$ be a RM problem that satisfies Assumption~\ref{as:rm}. 
    Fix any $T \geq 1$.
    For this problem, at each time step $t \in [T]$, $\mathbf{RM}$ takes as input the history $g_1, \dots, g_{t-1}$ and returns a decision $\theta_t \in \Theta$, with the following guarantee.
    For any sequence of loss vectors $g_1, \dots, g_T$, 
    \[
    \max_{\theta \in \Theta} \frac{1}{T} \sum_{t=1}^{T} \langle g_t, \theta_t - \theta \rangle \leq \frac{\mathbf{RM}\mathrm{Err}(B, L, T)}{T},
    \]
    where $\mathbf{RM}\mathrm{Err}(B, L, T) = o(T)$ is the oracle's promised rate.
\end{definition}

Going beyond worst-case regret, recent works have studied regret notions that are small when the realized sequence of losses has favorable structure. 
An example is optimistic regret.
On each round $t$, the Learner uses a prediction $m_t$ of the next loss vector $g_t$ alongside the history of loss vectors $g_1, \dots, g_{t-1}$.
The Learner's regret scales with the sum of prediction errors $\| g_t - m_t \|_2$, rather than the Lipschitz constant and time horizon. 
Such bounds capture the intuition that predictable losses should be easier to optimize. 

\begin{proposition}[Optimistic regret \citep{rakhlin_optimization_2013}]\label{prop:optim}
    Let $(\mathbb R^d, \Theta, \mathcal{G})$ be a RM problem that satisfies Assumption~\ref{as:rm}.
    Fix any $T \geq 1$.
    Let $m_1, \dots, m_T$ be a sequence of loss vector predictions with $\| m_t \|_2 \leq L$ for all $t \in [T]$.
    Then there exists an algorithm that uses $m_1, \dots, m_T$ to select a sequence of decisions $\theta_1, \dots, \theta_T \in \Theta$ such that, for any sequence of loss vectors $g_1, \dots, g_T \in \mathcal G$,
    \[
    \max_{\theta \in \Theta} \frac{1}{T} \sum_{t=1}^{T} \langle g_t, \theta_t - \theta \rangle \leq O \left( \frac{B\sqrt{\sum_{t=1}^{T} \| g_t - m_t \|^2_2}}{T} \right).
    \]
\end{proposition} 

Another example is strongly adaptive regret, which simultaneously controls worst-case regret on all intervals of the form $[t_1, t_2] \subset [1, T]$, with $t_1 < t_2$, and competes with the best-in-hindsight decision for each sub-interval. 

\begin{proposition}[Strongly adaptive regret \citep{jun_improved_2017}]\label{prop:sar}
    Let $(\mathbb R^d, \Theta, \mathcal{G})$ be a RM problem that satisfies Assumption~\ref{as:rm}.
    Fix any $T \geq 1$.
    Then there exists an algorithm that selects a sequence of decisions $\theta_1, \dots, \theta_T \in \Theta$ such that, for any sequence of loss vectors $g_1, \dots, g_T \in \mathcal G$,
    \[
    \max_{1 \leq t_1 < t_2 \leq T} \max_{\theta \in \Theta} \frac{1}{t_2-t_1+1} \sum_{s \in [t_1, t_2]} \langle g_s, \theta_s - \theta \rangle \leq O\left( \frac{BL(1+\sqrt{\log(t_2)})}{\sqrt{t_2 - t_1+1}} \right).
    \]
\end{proposition}

\subsection{Known Equivalences}\label{sec:equiv}
Next, we recall two reductions from \citet{abernethy_blackwell_2011} that establish an equivalence between Blackwell approachability and regret minimization.

\paragraph{Reducing RM to BA}

\begin{algorithm}[t]
\caption{RM with BA Oracle}\label{alg:ba-to-nr}
    \begin{algorithmic}[1]
    \REQUIRE BA Oracle $\mathbf{BA}$
    \STATE Initialize $f_0 = 0$ 
    \FOR{$t = 1,2,\dots$}
        \STATE Set approach direction $u_t = \mathbf{BA}(f_0, \dots, f_{t-1})$ 
        \STATE Play action $\theta_t = \theta(u_t)$ as defined in Equation~\ref{eq:rm-oh}
        \STATE Observe Nature's loss vector $g_t \in \mathcal G$
        \STATE Let $f_t = \frac{\langle g_t, \theta_t \rangle}{B} \oplus -g_t$
    \ENDFOR
    \end{algorithmic}
\end{algorithm}

In this direction, we are given a RM problem $(\mathbb R^d, \Theta, \mathcal G)$ that satisfies Assumption~\ref{as:rm} and solve it using a BA oracle. 
\citet{abernethy_blackwell_2011} showed that RM can be formulated as a conic approachability problem $(\mathbb R^{d+1}, A, B, f, S)$, where
\begin{equation}\label{eq:rm-setup}
    A = \Theta, \quad B = B_2(L), \quad f(\theta, g) = \frac{\langle g, \theta \rangle}{B} \oplus -g, \quad S = \mathrm{cone}(\{B\} \times \Theta)^\circ,
\end{equation}
and $B$ is the bound on the radius of the decision set $\Theta$.
The insight underlying this construction is that
\begin{align*}
    \max_{\theta \in \Theta} \frac{1}{T} \sum_{t=1}^{T} \langle g_t, \theta_t - \theta \rangle 
    &= \max_{\theta \in \Theta} \left\langle \frac{1}{T} \sum_{t=1}^{T} \left( \frac{\langle g_t, \theta_t \rangle}{B} \oplus -g_t \right), B \oplus \theta \right\rangle \\
    &\leq 2B \max_{\theta \in \Theta} \left\langle \frac{1}{T} \sum_{t=1}^{T} \left( \frac{\langle g_t, \theta_t \rangle}{B} \oplus -g_t \right), \frac{B \oplus \theta}{\| B \oplus \theta \|_2} \right\rangle \\
    &= 2B \max_{u \in S^\circ \cap B_2(1)} \left\langle \frac{1}{T} \sum_{t=1}^{T} \left( \frac{\langle g_t, \theta_t \rangle}{B} \oplus -g_t \right), u \right\rangle \\
    &= 2B \cdot d\left( \frac{1}{T}\sum_{t=1}^{T} f(\theta_t, g_t), S\right).
\end{align*}

Every $u \in S^\circ = \mathrm{cone}(\{B\} \times \Theta)^{\circ\circ} = \mathrm{cone}(\{B\} \times \Theta)$ takes the form $u = c(B \oplus \theta)$ for some non-negative $c$ and decision $\theta \in \Theta$.
Fix arbitrary $\theta_0 \in \Theta$. 
Then a valid halfspace oracle can be implemented as follows: 
\begin{equation}\label{eq:rm-oh}
    \theta(u) = 
    \begin{cases}
        \theta_0 &\quad \text{if} \quad u_1 = 0 \\
        \frac{B}{u_1}(u_2, \dots, u_{d+1}) &\quad \text{otherwise}
    \end{cases}
\end{equation}

\begin{proposition}[\citep{abernethy_blackwell_2011}]\label{prop:half}
    Given a RM problem $(\mathbb R^d, \Theta, \mathcal{G})$ that satisfies Assumption~\ref{as:rm}, define the BA problem $(\mathbb R^{d+1}, A, B, f, S)$ as in Equation~\ref{eq:rm-setup}.
    For any $u \in S^\circ$, playing $\theta(u)$ as in Equation~\ref{eq:rm-oh} ensures that for all $g \in \mathcal{G}$, $\langle u, f(\theta(u), g) \rangle = 0$.
\end{proposition}

The reduction is formalized in Algorithm~\ref{alg:ba-to-nr} and has the following guarantee.

\begin{proposition}[\citep{abernethy_blackwell_2011}]\label{prop:ba-to-nr}
    Given a RM problem $(\mathbb R^d, \Theta, \mathcal{G})$ that satisfies Assumption~\ref{as:rm}, define the BA problem $(\mathbb R^{d+1}, A, B, f, S)$ as in Equation~\ref{eq:rm-setup}.
    Fix any $T \geq 1$.
    Selecting the sequence of decisions $\theta_1, \dots, \theta_T \in \Theta$ using Algorithm~\ref{alg:ba-to-nr} guarantees that, for any sequence of loss vectors $g_1, \dots, g_T \in \mathcal{G}$,
    \[
    \max_{\theta \in \Theta}\frac{1}{T}\sum_{t=1}^{T} \langle g_t, \theta_t - \theta \rangle
    \leq 2B \cdot \frac{\mathbf{BA}\mathrm{Err}(\sqrt{2} L, T)}{T}. 
    \]
\end{proposition}

\paragraph{Reducing BA to RM}

\begin{algorithm}[t]
\caption{BA with RM Oracle}\label{alg:nr-to-ba}
    \begin{algorithmic}[1]
    \REQUIRE RM oracle $\mathbf{RM}$
    \STATE Initialize $g_0 = 0$
    \FOR{$t = 1,2,\dots$}
      \STATE Set approach direction $u_t = \mathbf{RM}(g_0,\dots,g_{t-1})$
      \STATE Play action $a_t = \mathcal{O}_H(u_t)$
      \STATE Observe payoff $f(a_t, b_t)$
      \STATE Set $g_t = -f(a_t, b_t)$
    \ENDFOR
  \end{algorithmic}
\end{algorithm}

In this direction, we are given a BA problem $(\mathbb R^d, A, B, f, S)$ that satisfies Assumption~\ref{as:ba}, as well as a valid halfspace oracle $\mathcal O_H$ for the problem, and solve it using a RM oracle. 
\citet{abernethy_blackwell_2011} first assume that $S$ is a cone and show that the problem of selecting approach directions be formulated as a RM problem $(\mathbb R^d, \Theta, \mathcal G)$, where
\begin{equation}\label{eq:ba-setup}
    \Theta = S^\circ \cap B_2(1), \quad \mathcal G = B_2(L),
\end{equation}
where the loss vector at time $t$ is written as $g_t = -f(a_t, b_t)$.
The key idea of this reduction is that at each time step $t \geq 1$,  approach direction $u_t$ can be mapped to an action $a_t = \mathcal O_H(u_t)$, and hence,
\begin{align*}
    \mathrm{dist}\left( \frac{1}{T} \sum_{t=1}^{T} f(a_t, b_t), S \right)
    &= \max_{u \in S^\circ \cap B_2(1)} \left\langle \frac{1}{T} \sum_{t=1}^{T} f(a_t, b_t), u \right\rangle \\
    &\leq \frac{1}{T} \sum_{t=1}^{T} \left\langle -f(a_t, b_t), u_t \right\rangle + \max_{u \in S^\circ \cap B_2(1)} \frac{1}{T} \sum_{t=1}^{T} \langle f(a_t, b_t), u \rangle \\
    &= \max_{u \in S^\circ \cap B_2(1)} \frac{1}{T} \sum_{t=1}^{T} \left\langle g_t, u_t - u \right\rangle
\end{align*}
where the inequality holds by the halfspace oracle guarantee. 
This reduction is formalized in Algorithm~\ref{alg:nr-to-ba} and has the following guarantee. 

\begin{proposition}[\citep{abernethy_blackwell_2011}]\label{prop:nr-to-ba}
    Let $(\mathbb R^d, A, B, f, S)$ be a BA problem that satisfies Assumption~\ref{as:ba} and further suppose $S$ is a cone.
    Define the RM problem $(\mathbb R^d, \Theta, \mathcal{G})$ as in Equation~\ref{eq:ba-setup}.
    Fix any $T \geq 1$.
    Selecting a sequence of actions $a_1, \dots, a_T \in A$ using Algorithm~\ref{alg:nr-to-ba} guarantees that, for any sequence of Nature's actions $b_1, \dots, b_T \in B$, 
    \[
     d \left( \frac{1}{T}\sum_{t=1}^{T} f(a_t, b_t), S \right) \leq \frac{\mathbf{RM}\mathrm{Err}(1, L, T)}{T}.
    \]
\end{proposition}

\citet{abernethy_blackwell_2011} show that if $(\mathbb R^d, A, B, f, S)$ satisfies Assumption~\ref{as:ba} but $S$ is not a cone, then the problem can be embedded into an auxiliary BA problem $(\mathbb R^{d+1}, A, B, \tilde f, \tilde S)$ that satisfies Assumption~\ref{as:ba}, where $\tilde S$ is a cone and the time-average of payoffs $\tilde f$ approaching $\tilde S$ implies that the time-average of payoffs $f$ approaches $S$. 
The following lemma formalizes this idea. 

\begin{proposition}[\citep{abernethy_blackwell_2011}]\label{prop:conic}
    Let $S \subset \mathbb{R}^d$ be a compact and convex set with $R = \max_{s \in S} \|s\|_2$, and fix any $x \notin S$.
    Let $\tilde x = R \oplus x$ and $\tilde S = \mathrm{cone}(\{ R \} \times S)$. 
    Then $d(\tilde x, \tilde S) \leq d(x, S) \leq 2d(\tilde x, \tilde S)$.
\end{proposition}

The auxiliary BA problem can be defined with
\begin{equation}\label{eq:ba-setup2}
\tilde f(a, b) = \{3L\} \oplus f(a, b), \quad \tilde S = \mathrm{cone}(\{ 3L\} \times (S \cap B_2(3L))).
\end{equation}

\begin{proposition}\label{prop:conic2}
    Let $(\mathbb R^d, A, B, f, S)$ be a BA problem that satisfies Assumption~\ref{as:ba}.
    Define the auxiliary BA problem $(\mathbb R^{d+1}, A, B, \tilde f, \tilde S)$ as in Equation~\ref{eq:ba-setup2} and use it to construct the RM problem $(\mathbb R^{d+1}, \Theta, \mathcal{G})$ as in Equation~\ref{eq:ba-setup}.
    Fix any $T \geq 1$.
    Selecting a sequence of actions $a_1, \dots, a_T \in A$ using Algorithm~\ref{alg:nr-to-ba} guarantees that, for any sequence of Nature's actions $b_1, \dots, b_T \in B$, 
    \[
    d \left( \frac{1}{T}\sum_{t=1}^{T} f(a_t, b_t), S \right) \leq \frac{\mathbf{RM}\mathrm{Err}(1, \sqrt{10}L, T)}{T}.
    \]
\end{proposition}

\section{Additional Material from Section 3}\label{sec:ba-to-geq-app}
\begin{algorithm}[t]
\caption{GEQ with BA Oracle}
\label{alg:ba-to-geq2}
    \begin{algorithmic}[1] 
    \REQUIRE Positive and increasing function $\alpha: \mathbb R_{+} \to \mathbb R_{+}$ and BA oracle $\mathbf{BA}$
    \STATE Initialize $s = 0$, $f_0 = 0$
    \FOR{$t = 1,2,\dots$}
    \IF{$s = t-1$}
        \STATE Set approach direction $u_t = \mathbf{BA}(f_0)$ 
    \ELSE
        \STATE Set approach direction $u_t = \mathbf{BA}(f_0, f_{s+1}, \dots, f_{t-1})$
    \ENDIF
    \IF{$u_t \neq 0$}
        \STATE Play action $\theta_t =  \alpha(t) \cdot \frac{u_t}{\|u_t\|_2}$
    \ELSE
        \STATE Play action $\theta_t = 0$
    \ENDIF
    \STATE Observe vector $g_t(\theta_t)$ and set $f_t = -g_t(\theta_t)$
    \IF{$\langle f_t, \theta_t \rangle > 0$}
        \STATE Update flag $s = t$
    \ENDIF
    \ENDFOR
  \end{algorithmic}
\end{algorithm}

\begin{theorem}\label{thm:ba-to-geq2}
    Suppose $(\mathbb R^d, \mathbb R^d, \mathcal G)$ is a GEQ problem that satisfies Assumption~\ref{as:restor}.
    Fix any $T \geq 1$.
    Then using Algorithm~\ref{alg:ba-to-geq2} to select the sequence of actions $\theta_1, \dots, \theta_T \in \mathbb R^d$ guarantees that, for any sequence of vector fields $g_1, \dots, g_T \in \mathcal G$, the inequality $\langle u_t, -g_t(\theta_t) \rangle \leq 0$ holds for all but $m \leq \lceil \alpha^{-1}(h) \rceil$ rounds, and 
    \[
    \left\| \frac{1}{T} \sum_{t=1}^{T} g_t(\theta_t) \right\|_2 
    \leq \frac{\mathbf{BA}\mathrm{Err}(L, T)}{T} + \frac{L \lceil \alpha^{-1}(h) \rceil}{T}
    \]
\end{theorem}

\begin{proof}[Proof of Theorem~\ref{thm:ba-to-geq2}]
    The variable $s$ in Algorithm~\ref{alg:ba-to-geq2} records the last round on which the BA oracle was reset.
    Concretely, at each round $t \in [T]$, Algorithm~\ref{alg:ba-to-geq2} queries the approachability oracle with the history of payoffs since the last time the oracle was restarted, $f_0, f_{s+1}, \dots, f_{t-1}$. 
    The oracle then returns approach direction $u_t$, which Algorithm~\ref{alg:ba-to-geq2} translates into an action with the rule $\theta_t = \alpha(t) \cdot \frac{u_t}{\|u_t\|_2}$. 
    After observing the realized vector $g_t(\theta_t)$, Algorithm~\ref{alg:ba-to-geq2} tests whether the restorativity condition failed. 
    That is, if $\langle -g_t(\theta_t), \theta_t \rangle > 0$, then $\| \theta_t \| = \alpha_t < h$ must hold, so the BA oracle is reset by updating $s = t$.
    This test fails on at most $\lceil \alpha^{-1}(h) \rceil$ and cannot fail on any round $t > \lceil \alpha^{-1}(h) \rceil$. 
    Overloading notation, let $s$ be the round with the final failure.
    Hence,
    \begin{align*}
        \left\| \frac{1}{T} \sum_{t=1}^{T} g_t(\theta_t) \right\|_2
        &\leq \left\| \frac{1}{T} \sum_{ t = s+1}^{T} g_t(\theta_t) \right\|_2 + \frac{1}{T} \sum_{t =1}^{s} \left\|  g_{t}(\theta_{t}) \right\|_2 \\
        &\leq \frac{\mathbf{BA}\mathrm{Err}(L, T-s)}{T} + \frac{L s}{T} \\
        &\leq \frac{\mathbf{BA}\mathrm{Err}(L, T)}{T} + \frac{L \lceil \alpha^{-1}(h) \rceil}{T}.
    \end{align*}
    The first line is by triangle inequality.
    The second line is by invoking the definition of a BA oracle to bound the first term, then invoking the $L$-boundedness of each $g_t$ to bound each term in the sum.
    The third line is due to the monotonicity of $\mathbf{BA}\mathrm{Err}(L, \cdot)$ in its second argument. 
\end{proof}
\end{document}